\documentclass[11pt]{article}

\usepackage{amsfonts}
\usepackage{amssymb}
\usepackage{amstext}
\usepackage{amsmath}
\usepackage{enumerate}
\usepackage{algorithm}
\usepackage{algorithmic}
\usepackage{amsmath}
\usepackage{xspace}
\usepackage{graphicx}
\usepackage{bbm}
\usepackage{tabularx}
\usepackage{makecell}

\usepackage{amsthm}

\usepackage{thmtools}

\usepackage{verbatim} 

\usepackage[usenames,dvipsnames,svgnames,table]{xcolor}
\definecolor{dartmouthgreen}{rgb}{0.05, 0.5, 0.06}
\definecolor{ceruleanblue}{rgb}{0.16, 0.32, 0.75}

\usepackage[colorlinks=true,pdfpagemode=UseNone,citecolor=dartmouthgreen,linkcolor=ceruleanblue,urlcolor=BrickRed,pdfstartview=FitH]{hyperref}

\usepackage{framed}
\usepackage{enumitem}
\usepackage{makecell}
\usepackage{multirow}


\newtheorem{theorem}{Theorem}[section]

\newtheorem{definition}[theorem]{Definition}
\newtheorem{assumption}[theorem]{Assumption}

\newtheorem*{problem*}{Problem}

\newtheorem*{remark*}{Remark}

\numberwithin{equation}{section}
\numberwithin{table}{section}

\renewcommand{\tilde}{\widetilde}

\newcommand{\E}[1]{{\mathbb{E}}\left[#1\right]}

\newcommand{\junk}[1]{}

\newcommand{\vertiii}[1]{{\left\vert\kern-0.25ex\left\vert\kern-0.25ex\left\vert #1 \right\vert\kern-0.25ex\right\vert\kern-0.25ex\right\vert}}

\newcommand{\one}{\ensuremath{\mathbbm{1}}}

\def\b1{{\bf 1}}

\def\U{{\cal U}}
\def\V{{\cal V}}

\global\long\def\E{\mathbb{E}}

\newcommand{\wtilde}{\widetilde}
\newcommand{\para}{\theta}
\newcommand{\distdata}{\wtilde{x}}
\newcommand{\origdata}{x}

\newcommand{\calD}{\mathcal{D}}

\newcommand{\calL}{\mathcal{L}}

\newcommand{\calA}{\mathcal{A}}
\newcommand{\calB}{\mathcal{B}}

\newcommand{\calM}{\mathcal{M}}

\newcommand{\calW}{\mathcal{W}}

\newcommand{\calO}{\mathcal{B}}

\usepackage{prettyref}
\usepackage{framed}
\usepackage[normalem]{ulem}
\newrefformat{lemma}{\savehyperref{#1}{Lemma~\ref*{#1}}}
\usepackage{afterpage}
\usepackage{makecell}
\usepackage{tabulary}
\usepackage{colortbl}
\usepackage{prettyref}
\usepackage{framed}

\newcommand{\pref}[1]{\prettyref{#1}}

\newcommand{\savehyperref}[2]{\texorpdfstring{\hyperref[#1]{#2}}{#2}}
\newrefformat{eq}{\savehyperref{#1}{Eq. \textup{(\ref*{#1})}}}
\newrefformat{eqn}{\savehyperref{#1}{Eq.~\textup{(\ref*{#1})}}}
\newrefformat{lem}{\savehyperref{#1}{Lemma~\ref*{#1}}}
\newrefformat{assump}{\savehyperref{#1}{Assumption~\ref*{#1}}}
\newrefformat{defi}{\savehyperref{#1}{\textbf{Definition}~\ref*{#1}}}
\newtheorem{defi}{Definition}[section]
\newrefformat{defi}{\savehyperref{#1}{Definition~\ref*{#1}}}\newrefformat{tab}{\savehyperref{#1}{Table~\ref*{#1}}}
\newrefformat{lemma}{\savehyperref{#1}{Lemma~\ref*{#1}}}
\newrefformat{line}{\savehyperref{#1}{Line~\ref*{#1}}}
\newrefformat{thm}{\savehyperref{#1}{Theorem~\ref*{#1}}}
\newrefformat{corr}{\savehyperref{#1}{Corollary~\ref*{#1}}}
\newrefformat{cor}{\savehyperref{#1}{Corollary~\ref*{#1}}}
\newrefformat{sec}{\savehyperref{#1}{Section~\ref*{#1}}}
\newrefformat{app}{\savehyperref{#1}{Appendix~\ref*{#1}}}
\newrefformat{ex}{\savehyperref{#1}{Example~\ref*{#1}}}
\newrefformat{fig}{\savehyperref{#1}{Figure~\ref*{#1}}}
\newrefformat{alg}{\savehyperref{#1}{Algorithm~\ref*{#1}}}
\newrefformat{rem}{\savehyperref{#1}{Remark~\ref*{#1}}}
\newrefformat{conj}{\savehyperref{#1}{Conjecture~\ref*{#1}}}
\newrefformat{prop}{\savehyperref{#1}{Proposition~\ref*{#1}}}
\newrefformat{proto}{\savehyperref{#1}{Protocol~\ref*{#1}}}
\newrefformat{prob}{\savehyperref{#1}{Problem~\ref*{#1}}}
\newrefformat{claim}{\savehyperref{#1}{Claim~\ref*{#1}}}
\newrefformat{que}{\savehyperref{#1}{Question~\ref*{#1}}}
\newrefformat{op}{\savehyperref{#1}{Open Problem~\ref*{#1}}}
\newrefformat{fn}{\savehyperref{#1}{Footnote~\ref*{#1}}}
\newrefformat{property}{\savehyperref{#1}{Property~\ref*{#1}}}
\newtheorem{property}{Property}
\newtheorem{thm}{Theorem}[section]

\newtheorem{lem}[thm]{Lemma}

\makeatletter
\DeclareRobustCommand{\cev}[1]{%
  {\mathpalette\do@cev{#1}}%
}
\newcommand{\do@cev}[2]{%
  \vbox{\offinterlineskip
    \sbox\z@{$\m@th#1 x$}%
    \ialign{##\cr
      \hidewidth\reflectbox{$\m@th#1\vec{}\mkern4mu$}\hidewidth\cr
      \noalign{\kern-\ht\z@}
      $\m@th#1#2$\cr
    }%
  }%
}
\makeatother

\title{A Unified Learn-to-Distort-Data Framework for Privacy-Utility Trade-off in Trustworthy Federated Learning}

\author{
  Xiaojin Zhang\textsuperscript{\rm 1} Mingcong Xu\textsuperscript{\rm 1} Wei Chen\textsuperscript{\rm 1}\\
  \textsuperscript{\rm 1} Huazhong University of Science and Technology\\ 
  }

\date{}

\begin{document}

\begin{titlepage}
\def\thepage{}
\thispagestyle{empty}

\maketitle

\begin{abstract}
 In this paper, we first give an introduction to the theoretical basis of the privacy-utility equilibrium in federated learning based on Bayesian privacy definitions and total variation distance privacy definitions. We then present the \textit{Learn-to-Distort-Data} framework, which provides a principled approach to navigate the privacy-utility equilibrium by explicitly modeling the distortion introduced by the privacy-preserving mechanism as a learnable variable and optimizing it jointly with the model parameters. We demonstrate the applicability of our framework to a variety of privacy-preserving mechanisms on the basis of data distortion and highlight its connections to related areas such as adversarial training, input robustness, and unlearnable examples. These connections enable leveraging techniques from these areas to design effective algorithms for privacy-utility equilibrium in federated learning under the \textit{Learn-to-Distort-Data} framework.
\end{abstract}

\end{titlepage}

\thispagestyle{empty}

\newpage

\section{Introduction}\label{sec: introduction}
Federated learning has evolved as a hopeful paradigm for enabling collaborative learning among various parties while maintaining data privacy \cite{mcmahan2017communication, kairouz2019advances}. In federated learning, numerous clients combine efforts to train a joint model without directly disclosing their private information. Instead, every client uses its personal data to train an independent model, communicating only the model modifications with the central server, which aggregates these updates to ameliorate the global model. This strategy allows clients to benefit from the collective information of the whole federation without compromising the confidentiality of their individual datasets.

However, relevant studies have manifested that federated learning is susceptible to privacy attacks, where an adversary can infer important details about the parties' private data by analyzing the model modifications shared\cite{melis2019exploiting, nasr2019comprehensive}. To mitigate this vulnerability, various privacy-preserving mechanisms have been proposed, including differential privacy \cite{dwork2006calibrating, abadi2016deep}, secure multiparty computation \cite{bonawitz2017practical}, and homomorphic encryption \cite{hardy2017private}. These mechanisms aim to preserve the confidentiality of the parties' data by mixing noise or encryption into the model updates prior to communicating them with the server.

While these privacy-preserving mechanisms have shown promise in enhancing the confidentiality of federated learning systems, they often result in decreased model utility or usefulness. The inherent equilibrium between privacy and utility poses significant challenges in designing productive and efficient mechanisms for preserving privacy in federated learning. On one hand, strong privacy guarantees are essential to safeguard the confidence of the clients' details and maintain the trust of the participants. On the other hand, high model utility is crucial to make sure that the practicality and usefulness of the learned models in practical implementations.

Within this paper, we first manifest the theoretical basis of the privacy-utility equilibrium in federated learning based on Bayesian privacy definitions and total variation distance privacy definitions. We present quantitative relationships between privacy leakage and utility loss under these privacy definitions, which provide valuable insights into the inherent trade-offs in privacy-preserving federated learning.
We then propose the \textit{Learn-to-Distort-Data} framework, which provides a principled approach to navigate the privacy-utility equilibrium by explicitly modeling the distortion introduced by the privacy-preserving mechanism as a learnable variable and optimizing it jointly with the model parameters. The framework formulates the privacy-preserving federated learning issue as a constrained optimisation issue, with the target of minimizing the utility loss while ensuring that the privacy disclosure remains within an acceptable threshold.

We demonstrate the applicability of the \textit{Learn-to-Distort-Data} framework to a wide range of privacy-preserving mechanisms based on data distortion, including differential privacy, secure multiparty computation, homomorphic encryption, and data compression. By properly designing the distortion variable and the loss function, the framework can capture the essential features of different mechanisms and optimize their performance in a unified manner.

Furthermore, we highlight the connections between the \textit{Learn-to-Distort-Data} formulation and related areas such as adversarial training, input robustness, and unlearnable examples. Adversarial training aims to improve the resistance of machine learning models against adversarial attacks by expanding the data with adversarially perturbed samples. Input robustness refers to a model's ability to stay its property facing input perturbations or noise. Unlearnable examples are input examples that are difficult or impossible for a model to learn from, which can arise due to various factors such as label noise, data corruption, or adversarial attacks. We show that the techniques and insights from these areas can be readily adapted to our \textit{Learn-to-Distort-Data} framework to design effective algorithms for privacy-utility equilibrium in federated learning.
The following is a summary of our paper's principal contributions:
\begin{itemize}
\item We introduce the theoretical foundations of the privacy-utility equilibrium in federated learning based on Bayesian privacy definitions and total variation distance privacy definitions, providing quantitative relationships between privacy leakage and utility loss.
\item We propose the \textit{Learn-to-Distort-Data} framework, which provides a principled approach to navigate the privacy-utility equilibrium by explicitly modeling the distortion introduced by the privacy-preserving mechanism as a learnable variable and optimizing it jointly with the model parameters.
\item We discuss the applicability of our framework to multiple kinds of  privacy-preserving mechanisms based on data distortion, including differential privacy, secure multiparty computation, homomorphic encryption, and data compression.
\item We highlight the connections between the \textit{Learn-to-Distort-Data} formulation and related areas such as adversarial training, input robustness, and unlearnable examples, enabling leveraging techniques from these areas to design effective algorithms for privacy-utility equilibrium in federated learning.
\end{itemize}
The remaining part in this paper is structured in the following order. In \pref{sec: preliminaries}, we provide the fundamental concepts and notations that are essential for understanding our \textit{Learn-to-Distort-Data} framework. In \pref{sec:PU-tradeoff}, we give an introduction to the theoretical foundations of the privacy-utility equilibrium in federated learning based on Bayesian privacy definitions. In \pref{subsect:leak-GJS}, we propose the theoretical foundations of the privacy-utility equilibrium in federated learning based on $\alpha$-skew Jensen-Shannon Divergence privacy definitions. In \pref{subsec:thm-tv}, we introduce the privacy-utility equilibrium based on total variation distance privacy definitions. In \pref{sec: learn_to_distort}, we present our \textit{Learn-to-Distort-Data} framework and demonstrate its applicability to various privacy-preserving mechanisms based on data distortion. In \pref{sec: related_areas}, we give a discussion of the connections between the \textit{Learn-to-Distort-Data} formulation and related areas, and highlight how these connections can be leveraged to design effective algorithms for privacy-utility equilibrium in federated learning. Finally,  we give the conclusion of the whole article in \pref{sec: conclusion}.

\section{Related Work}

This section investigates relevant research on privacy-utility trade-off in the aspects of machine learning, federated learning, and then introduce adversarial training and unlearnable examples, which are closely related to our proposed \textit{Learn-to-Distort-Data} framework.

\subsection{Advancing Privacy-Utility Trade-offs in Machine Learning and Federated Learning}

The trade-off of privacy-utility has been a fundamental challenge in various domains, including machine learning and federated learning. This section begins with a review of relevant research on the trade-off of privacy-utility about the general machine learning, and then focus on the specific setting of federated learning.

\subsubsection{Privacy-Utility Trade-offs in Machine Learning}

In the machine learning literature, the trade-off of privacy-utility has been researched from diverse perspectives. For instance, in locally private contexts, \cite{duchi2013local} examined how to balance privacy and convergence rate. They showed the convergence rate of learning algorithms can be maintained while ensuring a certain level of privacy protection. \cite{rassouli2019optimal} proposed an ideal trade-off of privacy-utility via linear programming. However, their closed-form solution is only suitable for a particular binary scenario.

The rate-distortion-uncertainty region, which describes the ideal trade-off among compression rate, distortion, and privacy leakage, has been explored in the context of information theory \cite{reed1973information, yamamoto1983source, sankar2013utility}. \cite{sankar2013utility} quantified utility using privacy and accuracy using entropy, and defined a utility-privacy balance region for independent and identically distributed datasets by leveraging rate-distortion theory. However, extending these results to more general scenarios remains an open challenge.

\cite{wang2016relation} proposed a united privacy-distortion framework, in which distortion was measured by computing the anticipated distance in Hamming units between the input and output. They assessed privacy leaks by differential privacy (DP), identifiability, and reciprocal information as separate measurements, and established the connection between these different privacy indicators.

\subsubsection{Balancing Privacy and Utility in Federated Learning}

In light of federated learning, the trade-off of privacy-utility has obtained significant attention as a result of the dispersed character of the learning process and the need to defend the privacy of the clients' data. \cite{zhang2022no, zhang2023towards, zhang2023trading, zhang2023meta, zhang2023probably, zhang2023game, zhang2023theoretically} have investigated the trade-off of privacy-utility about federated learning by considering distortion on the model parameters. They formulate this trade-off as a constrained optimisation issue, where the objective is to reduce the loss in utility while satisfying a predefined constraint on privacy leakage.

In contrast to these works, our learn-to-distort framework focuses on the trade-off of privacy-utility about federated learning by introducing distortion on the data itself. By explicitly modeling the distortion on the data as a learnable variable, our framework provides a principled approach to navigate the trade-off of privacy-utility in federated learning. This distinction highlights the novelty of our work compared to the existing literature on the trade-off of privacy-utility in the aspect of federated learning.

Our learn-to-distort framework unifies various techniques for preserving the private, including differential privacy, secure multiparty computation, and homomorphic encryption, under a common optimization problem. By properly designing the distortion variable and the loss function, our framework can capture the essential features of different mechanisms and optimize their performance in a unified manner. This unified perspective on privacy-preserving federated learning distinguishes our work from the existing studies that focus on specific privacy-preserving mechanisms.

Furthermore, our theoretical analysis of the trade-off of privacy-utility under the learn-to-distort framework provides valuable insights for the design and evaluation of mechanisms for preserving privacy in federated learning environments. By establishing a fundamental relationship between the privacy leakage and the distortion extent, our work lays the foundation for developing principled approaches to acquire the ideal equilibrium between privacy protection and model property in federated learning.

Table \ref{tab: related_work} summarizes the key contributions of the related work on the trade-off of privacy-utility in both machine learning and federated learning contexts.

\begin{table}[h]
\small
\centering
\caption{Related work on the trade-off of privacy-utility in machine learning and federated learning.}
\label{tab: related_work}
\renewcommand{\arraystretch}{1.5} 
\begin{tabular}{|>{\centering\arraybackslash}m{0.25\textwidth}|>{\centering\arraybackslash}m{0.35\textwidth}|>{\centering\arraybackslash}m{0.35\textwidth}|}
\hline
\thead{Reference} & \thead{Context} & \thead{Key Contributions} \cr
\hline
\cite{duchi2013local} & Machine Learning & Balance between privacy and convergence rate \cr
\hline
\cite{rassouli2019optimal} & Machine Learning & Optimal trade-off via linear programming \cr
\hline
\makecell{\cite{reed1973information} \\ \cite{yamamoto1983source} \\ \cite{sankar2013utility}} & Machine Learning & Rate-distortion-equivocation region \cr
\hline
\cite{sankar2013utility} & Machine Learning & Utility-privacy balance region \cr
\hline
\cite{wang2016relation} & Machine Learning & Relationship between different privacy metrics \cr
\hline
\makecell{\cite{zhang2022no, zhang2023towards}\\ \cite{zhang2023trading, zhang2023meta}\\ \cite{zhang2023probably, zhang2023game, zhang2023theoretically}} & Federated Learning & \makecell{The trade-off of privacy-utility with\\ parameter distortion} \cr
\hline
Our work & Federated Learning & The trade-off of privacy-utility with data distortion \cr
\hline
\end{tabular}
\end{table}

In summary, our learn-to-distort framework boosts privacy-preserving federated learning to a new level through a unified approach to navigate the trade-off of privacy-utility through data distortion. Our work complements and extends the existing literature on the trade-off of privacy-utility in both machine learning and federated learning contexts, and opens up new avenues for future research in this important area.

\subsection{Adversarial Training}\label{subsec: adversarial_training}

A formula called adversarial training aims to strengthen the robustness of models used for machine learning  against adversarial attacks. The central idea is to extend the training data with adversarially perturbed samples and train the model to accurately categorize these examples. One way to formula optimization problem for adversarial training is as a min-max issue:

\begin{align}\label{eq: adv_train_opt_problem_reformulated}
\begin{array}{r@{\quad}l@{}l@{\quad}l}
\quad\min_{\para}&\max_{\distdata}\calL(f(\para; \distdata), y),\\
\text{s.t.,} & \|\distdata - \origdata\|\le\epsilon.
\end{array}
\end{align}

Here, $\para$ represents the model parameters, $\distdata$ denotes the perturbed input samples, $\origdata$ means the initial input samples, $y$ is the matching label, and $\calL(\cdot)$ denotes the loss function. The constraint $\|\distdata - \origdata\|\le\epsilon$ ensures that the perturbations are bounded within a certain range $\epsilon$ to maintain the similarity between the perturbed and original samples.

The optimization problem can be further reformulated in terms of the perturbation variable $\delta$:

\begin{align}\label{eq: adv_train_opt_problem_perturbation}
\begin{array}{r@{\quad}l@{}l@{\quad}l}
\quad\min_{\para}&\max_{\delta}\calL(f(\para; x + \delta), y),\cr
\text{s.t.,} & \|\delta\|\le\epsilon.
\end{array}
\end{align}

In this formulation, $\delta$ represents the perturbation added to the original input $x$. The constraint $\|\delta\|\le\epsilon$ limits the magnitude of the perturbation to ensure that the perturbed samples remain within a certain distance from the original samples.

Adversarial training has been widely studied in the literature \cite{goodfellow2014explaining, madry2017towards} and is useful to boost the resilience of models used for machine learning to a number of adversarial assaults. The \textit{Learn-to-Distort-Data} framework proposed in this paper shares some similarities with adversarial training, as both aim to find the optimal perturbation or distortion that optimizes a certain objective function while satisfying certain constraints. As such, the techniques and insights from adversarial training can be adapted to generate privacy-preserving distortions and optimize the trade-off of privacy-utility in federated learning systems.

\subsection{Unlearnable Examples}

Based on the research of adversarial training, other methods to enhance robustness have been put forward already. Including the study of unlearnable samples \cite{huang2021unlearnable}, by adding disturbance to the training data, in order to prevent the model from studying effective features, so as to defend the privacy of data; It is also possible to introduce data disturbance to make the model produce the wrong classification effect, and to boost the robustness of the neural network model using error minimization attacks \cite{zheng2020evaluating}.

The optimization objective of the error minimization attack problem is as follows:
\begin{equation}
\begin{aligned}
\min_{\theta }\max_{\delta } & \quad \mathbb{E}_{(x, y) \sim T}[\mathcal{L}(f(x + \delta; \theta), y)], \\
\text{s.t.} & \quad \|\delta\| \leq \epsilon.
\end{aligned}
\end{equation}
Here, ${y}$ is the true label, $(x, y)$ represents the training data point, and $T$ symbolizes the distribution of the training data set. By introducing disturbance to the training data, the optimization problem maximizes the model's loss, thereby increasing the unlearnability of the data and protecting the data privacy.

The optimization objective of the error minimization attack problem is as follows:
\begin{equation}
\begin{aligned}
\min_{\theta }\min_{\delta } & \quad \mathbb{E}_{(x, y) \sim T}[\mathcal{L}(f(x + \delta; \theta), y)], \\
\text{s.t.} & \quad \|\delta\| \leq \epsilon.
\end{aligned}
\end{equation}
 The central purpose of the attacker is to add subtle perturbations to the training sample, while minimizing the model's classification errors in the training data, thus the model can't be correctly classified during the training process, thus reducing the model's correctness on the test data.

\section{Preliminaries}\label{sec: preliminaries}

This section presents the fundamental premises and notations that are essential for understanding our learn-to-distort framework. We first provide an overview of federated learning and its vulnerability to privacy attacks. We then formally define the semi-honest attacker in federated learning, and introduce the privacy leakage and utility loss metrics used throughout the paper. In addition, we give a detailed notation table in \pref{sec: notations table}.

\subsection{Federated Learning}

Federated learning is a collaborative learning paradigm that permits numerous clients to train a joint model without directly providing their private information \cite{mcmahan2017communication}. In a typical federated learning setup, every client $k \in \{1, \dots, K\}$ has a personal dataset $\mathcal{D}_k = \{(x_i, y_i)\}_{i=1}^{n_k}$, where $x_i \in \mathbb{R}^d$ represents the feature vector and $y_i \in \mathbb{R}$ denotes the corresponding label. The purpose is to acquire a global model $f_{\theta}: \mathbb{R}^d \rightarrow \mathbb{R}$ parameterized by $\theta$ that minimizes the following objective function:
\begin{equation}\label{eq: federated_learning_objective}
\min_{\theta} \frac{1}{K} \sum_{k=1}^K \frac{1}{n_k} \sum_{i=1}^{n_k} \mathcal{L}(f_{\theta}(x_i), y_i),
\end{equation}
where the loss function is symbolized by $\mathcal{L}(\cdot, \cdot)$.

The federated learning procedure can be presented as follows:
\begin{enumerate}
\item Every client $k$ trains a model locally $\theta_k$ on their personal dataset $\mathcal{D}_k$ by minimizing the local objective function:
\begin{equation}\label{eq: local_objective}
\min_{\theta_k} \frac{1}{n_k} \sum_{i=1}^{n_k} \mathcal{L}(f_{\theta_k}(x_i), y_i).
\end{equation}
\item Every client $k$ shares the local model modifications $\Delta \theta_k = \theta_k - \theta$ with a central server, where $\theta$ represents the current global model.
\item The server modifies the global model by combining the local modifications from all clients:
\begin{equation}\label{eq: global_update}
\theta \leftarrow \theta + \frac{1}{K} \sum_{k=1}^K \Delta \theta_k.
\end{equation}
\item Steps 1-3 are repeated for multiple rounds until the global model converges or a ideal level of performance is achieved.
\end{enumerate}

While federated learning provides a promising approach for enabling collaborative learning while preserving data privacy, it is susceptible to privacy attacks, as we discuss next.

\subsection{Semi-honest Attacker in Federated Learning}

The semi-honest attacker of federated learning is an adversary who follows the regulations correctly but tries to infer sensitive info about the clients' private information by studying the model modifications shared. We imagine a scenario in which the attacker has entry to the global model parameter $\theta$ and aims to conclude the initial dataset $\breve{d}$ of a specific user.

The attacker can formulate an optimization problem to conclude the initial dataset on the basis of the uncovered model parameter. Let $d$ denotes the inferred dataset, and let $g(d) = \frac{\partial \mathcal{L} (\theta, d)}{\partial \theta}$ be a function that maps the dataset $d$ to the corresponding model parameter. The attacker's objective is to seek a dataset $d$ that minimizes the deviation between the inferred parameter $g(d)$ and the exposed model parameter $\theta$:
\begin{equation}\label{eq: attacker_objective}
\min_{d} \| g(d) - \theta \|^2.
\end{equation}

By solving this optimization problem, the attacker can obtain an inferred dataset $d$ that is not far from the original dataset $\breve{d}$ in respect of the model parameter. The quality of the inferred dataset determines whether the assault is successful, which can be measured by the privacy leakage metric, as we define next.

\section{Trade-off between Bayesian Privacy and Model Utility}\label{sec:PU-tradeoff}

This section provide a theoretical examination of the privacy-utility trade-off within federated learning when Bayesian privacy is measured using $\alpha$-skewed Jensen-Shannon divergence and total variation distance. We first introduce the necessary notations and definitions.

$F^{\calA}_k$, $F^{\calO}_k$, and $F^{*}_k$ symbolize the attacker's faith distribution about the client's private data $D_k$ upon viewing the safeguarded info, without observing any info, and viewing the unprotected content, respectively. The corresponding probability density functions are denoted as $f^{\calA}$, $f^{\calO}$, and $f^{*}$. The model parameters are represented by $w$, and the personal data from client $k$ is denoted as $D_k$.

\begin{definition}[Utility Loss]\label{defi: utility_loss}
The definition of utility loss is
\begin{align*}
   \epsilon_{u} =  \frac{1}{K}\sum_{k=1}^K \epsilon_{u, k} = \frac{1}{K}\sum_{k=1}^K [U(P^{o}, D_k) - U(P^d, D_k)],
\end{align*}
where $U(P, D_k) = \mathbb E_{w_k\sim P}\frac{1}{|D_k|}\sum_{d\in D_k}U(w_k, d)$.
\end{definition}

This utility loss measures the variation in model performance across the unprotected model information $P^{o}$ and the protected model information $P^d$, averaged over all clients.

\begin{definition}[Majority Gap]
Assume that the utility resulted from sharing any parameter in $\mathcal W^*_k$ achieves the maximal value. Specifically, for any $w\in\mathcal W^*_k$, it holds that

\begin{align*}
    U(w, D_k) = \max_{w'\in\mathcal W_k} U(w', D_k).
\end{align*}
For each client $k$, let $\Delta_k$ be the maximum constant satisfying that
\begin{align}
   \int_{\calW_{\Delta_k}} p^{d}_{W_k}(w) dw\le\frac{||P_a^{d} - P_a^{o}||_{\text{TV}}}{2},
\end{align}
where $\calW_{\Delta_k} = \{w\in\calW^d_k: |U(w, D_k) - U(w^{*}, D_k)|\le\Delta_k\}$, $w^{*}\in\mathcal W^*_k$, and $\Delta_k$ depends on the protection mechanism $\mathcal M_k$, the utility function $U$, and the data set $D_k$.  
\end{definition}

The majority gap $\Delta_k$ symbolizes the maximum utility distinction across the protected model parameters and the optimal model parameters for client $k$, such that the probability of sampling a protected model parameter with utility close to the optimal is confined by the total variation distance across the protected and unprotected model distributions.

\begin{assumption}[Positive Majority Gap]\label{assump: defi_of_Delta}
Assume that $\Delta_k>0$ for each client $k$.
\end{assumption}

This assumption ensures that the majority gap is positive for all clients, which is necessary for deriving the quantitative relationships between utility loss and  privacy leakage.

Now we provide a reduced limit on the utility loss under the scenarios when privacy is measured using $\alpha$-skewed Jensen-Shannon divergence and total variation distance. These lower bound also illustrate the quantitative relationships across utility loss and privacy leaks. 

\subsection{Measuring Privacy Leakage using $\alpha$-skew Jensen-Shannon Divergence}\label{subsect:leak-GJS}

We proceed to study a generalized Jensen-Shannon divergence using the skewness parameter $\alpha$.

\begin{definition}[Generalized Jensen-Shannon Divergence]\cite{nielsen2020generalization}
   Let $f^{\calM_\alpha}_{D_k}(d) = \alpha f^{\calA}_{D_k}(d) + (1-\alpha)f^{\calO}_{D_k}(d)$ represent the mixture density of distribution $F^{\calM_\alpha}_k$, we define
\begin{align}
JS_{\alpha}(F^{\calA}_k || F^{\calO}_k) = \alpha KL(F^{\calA}_k || F^{\calM_\alpha}_k) + (1-\alpha)KL(F^{\calO}_k || F^{\calM_\alpha}_k),
\end{align}
where $\alpha$ represents the preference of smoothing between two distributions, and offers the ability to adjust the divergence skew.  \cite{deasy2020constraining}.
\end{definition}

Jensen-Shannon divergence in generalized allows a skewed mixing of the two distributions, controlled by the parameter $\alpha$. When $\alpha = 1/2$, it reduces to the standard Jensen-Shannon divergence.

\begin{definition}[Bayesian Privacy Leakage with $\alpha$-skew JS Divergence]\label{defi: average_privacy_alpha_JSD}
The  client $k$'s average privacy leakage is identified as
\begin{align}\label{eq: def_of_apl_alpha_JS}
\epsilon_{p, k} = [JS_{\alpha}(F^{\calA}_k || F^{\calO}_k)]^{1/e}.
\end{align}
\end{definition}

This definition quantifies the privacy leakage by the $\alpha$-skewed Jensen-Shannon divergence between the attacker's faith distributions with and without viewing the safeguarded info.

Prior to going into great detail about the analysis, let's recall the salient features of standard JS divergence.
\begin{property}\label{property: TriangleiIneq}
The triangle inequality is satisfied by the Jensen-Shannon divergence square root. Specifically, 
    \begin{align*}
        \sqrt{JS(F^{*}_k || F^{\calO}_k)} - \sqrt{JS(F^{\calA}_k || F^{\calO}_k)}\le \sqrt{JS(F^{\calA}_k || F^{*}_k)}.
    \end{align*}
\end{property}

\begin{property}\label{property: derive_using_AM-GM_inequality}
From AM-GM inequality, we have
\begin{align*}
    \frac{f^{\calA}_{D_k}(d)}{f^{\calM}_{D_k}(d)}\le \frac{f^{\calM}_{D_k}(d)}{f^{\calO}_{D_k}(d)}.
\end{align*}
\end{property}

When $\alpha = 1/2$, $JS_{\alpha}$ reduces to $JS$ divergence. Given the maximum privacy leakage, we derive a lower bound for the utility loss using $JS_{\alpha}$.

The following lemma illustrates that $JS_{\alpha}^{1/e}$ satisfies the triangle inequality, which follows directly from Lemma 4.7 of \cite{chen2008metrics}. \pref{property: TriangleiIneq} is generalized as follows, which holds for any $\alpha\in [0,1]$.


\begin{lem}\label{lem: triangle_inequality_mt}
For $\alpha$-skew Jensen-Shannon Divergence, it holds that
\begin{align*}
    JS_{\alpha}(F^{*}_k || F^{\calO}_k)^{1/e}\le JS_{\alpha}(F^{\calA}_k || F^{\calO}_k)^{1/e} + JS_{\alpha}(F^{\calA}_k || F^{*}_k)^{1/e}.
\end{align*}
\end{lem}

This lemma states that the triangle inequality is satisfied by the $\alpha$-skewed Jensen-Shannon divergence's $e$-th root, which is crucial for deriving the quantitative relationship between privacy leakage and utility loss.

We note that \pref{property: derive_using_AM-GM_inequality} does not hold for any $\alpha\in [0,1]$. In other words, $\frac{f^{\calA}_{D_k}(d)}{f^{\calM_\alpha}_{D_k}(d)}\le \frac{f^{\calM_\alpha}_{D_k}(d)}{f^{\calO}_{D_k}(d)}$ only holds when $\alpha = 1/2$. This prevents us to provide a quantitative relationship following the analysis when privacy leakage is measured using JS divergence. To circumvent this issue, we directly build a link across the generalized JS divergence and the KL divergence. Specifically,
\begin{align}\label{eq: general_beta_alpha_mt}
    JS_{\alpha}(F^{\calA}_k || F^{*}_k)\le \max\{2\alpha KL(F^{\calA}_k || F^{\calM_\alpha}_k), 2(1-\alpha)KL(F^{*}_k || F^{\calM_\alpha}_k)\}.
\end{align}

The RHS of \pref{eq: general_beta_alpha_mt} is upper bounded by $\max_{d\in\mathcal D}\left\{2\alpha\left|\log\frac{f^{\calA}_{D_k}(d)}{f^{\calM_\alpha}_{D_k}(d)}\right|, 2(1-\alpha)\left|\log\frac{f^{*}_{D_k}(d)}{f^{\calM_\alpha}_{D_k}(d)}\right|\right\}$ from the definition of KL divergence, which is further upper bounded by $2\alpha(1-\alpha)(e^{2\delta}-1)||P^{o}_k - P^{d}_k||_{\text{TV}}$.

Kindly consult \pref{sec: generalized_JSD} for a comprehensive analysis.

\begin{thm}[Quantitative Balance between Privacy and Utility I]\label{thm: utility-privacy trade-off_GJSD_mt}
Assuming \pref{assump: defi_of_Delta} is correct, $\epsilon_{u}$ and $\epsilon_{p, k}$ have been identified in \pref{defi: utility_loss} and \pref{defi: average_privacy_alpha_JSD}. Then 
\begin{align*}
   \frac{1}{K}\sum_{k=1}^K JS_{\alpha}(F^{*}_k || F^{\calO}_k)^{1/e}\le\frac{1}{K}\sum_{k=1}^K \epsilon_{p, k} + \frac{1}{K}\sum_{k=1}^K \left[2\alpha(1-\alpha)(e^{2\delta}-1)||P^{o}_k - P^{d}_k||_{\text{TV}}\right]^{1/e},
\end{align*}
where $\delta = \max_{w\in \mathcal{W}_k, d \in \mathcal D_k} \left|\log\left(\frac{f_{D_k|W_k}(d|w)}{f_{D_k}(d)}\right)\right|$ symbolizes the maximum privacy leaks across all possible information $w$ released by client $k$, which is a constant unaffected by the safeguarding method. \\
Furthermore, based on \pref{assump: defi_of_Delta}, we have that

\begin{align*}
 \frac{1}{K}\sum_{k=1}^K JS_{\alpha}(F^{*}_k || F^{\calO}_k)^{1/e}\le\overline\epsilon_{p} + 2\gamma[2\alpha(1-\alpha)(e^{2\delta}-1)]^{1/e}\epsilon_u/\overline\Delta,
\end{align*}
where $\gamma = \frac{\frac{1}{K}\sum_{k=1}^K ||P^{o}_k - P^d_k||^{1/e}_{\text{TV}}}{||P^d - P^{o}||_{\text{TV}}}$, $\overline\Delta = \sum_{k=1}^K \Delta_k$, and $\overline\epsilon_{p} = \frac{1}{K}\sum_{k=1}^K \epsilon_{p, k}$.
\end{thm}

This theorem establishes a quantitative relationship between the privacy leakage measured by the $\alpha$-skewed Jensen-Shannon divergence and the utility loss. It shows that the privacy leakage can be bounded with a function depending on the utility loss, the total variation distance across the protected and unprotected model distributions, and the skewness parameter $\alpha$. The implication is that reducing the utility loss and the difference across the safeguarded and unguarded model distributions can help limit the privacy leakage.

Please refer to \pref{sec: generalized_JSD} for detailed analysis.

\subsection{Measuring Privacy Leakage Using Total Variation Distance} \label{subsec:thm-tv}
In this scenario, the privacy disclosure of any information $w$ is  identified as $\epsilon_{p, w} = |f_{D_k|W_k}(d|w) - f_{D_k}(d)|$. The average privacy disclosure is then calculated by taking average over all possible assignments of $d$ and $w$.
\begin{definition}[Bayesian Privacy Leakage with Total Variation Distance]\label{defi: average_privacy_dtv}
The average privacy leakage is defined as
\begin{align}
    \epsilon_{p, k} = \int_{d\in\mathcal D_k}p_{D_k}(d)|f^{\calA}_{D_k}(d) - f_{D_k}(d)|d\mu(d),
\end{align}
where $f^{\calA}_{D_k}(d) = \int_{\mathcal{W}_k} f_{D_k|W_k}(d|w)dP^d_k(w)$.
\end{definition}

This definition quantifies the privacy leakage with the total variation distance across the attacker's faith distribution with and without viewing the safeguarded information, averaged over all possible private data samples.

\begin{thm}[Quantitative Balance between Privacy and Utility II]\label{thm: utility-privacy trade-off_TVD_mt}
Assuming \pref{assump: defi_of_Delta} is correct, $\epsilon_{u}$ and $\epsilon_{p, k}$ have been identified in \pref{defi: utility_loss} and \pref{defi: average_privacy_dtv}. Then 
\begin{align*}
    \frac{1}{K}\sum_{k=1}^K ||F^{*}_k - F^{\calO}_k||_{\text{TV}}\le \frac{1}{K}\sum_{k=1}^K \epsilon_{p, k} +  \frac{1}{K}\sum_{k=1}^K 2\delta||P^{o}_k - P^d_k||_{\text{TV}},
\end{align*}
where $\delta = \max_{w\in \mathcal{W}_k, d \in \mathcal D_k} |f_{D_k|W_k}(d|w) - f_{D_k}(d)|$ symbolizes the maximum privacy disclosure over all possible information $w$ released by client $k$, which is a constant unaffected by the safeguarding method. \\
Furthermore, based on \pref{assump: defi_of_Delta}, we have that
\begin{align*}
    \frac{1}{K}\sum_{k=1}^K \sqrt{JS(F^{*}_k || F^{\calO}_k)} \le\overline\epsilon_{p} + \frac{\gamma}{4\overline\Delta}(e^{2\delta}-1)\epsilon_u,
\end{align*}
where $\gamma = \frac{\frac{1}{K}\sum_{k=1}^K ||P^{o}_k - P^d_k||_{\text{TV}}}{||P^d_a - P^{o}_a||_{\text{TV}}}$, $\overline\Delta = \sum_{k=1}^K \Delta_k$, and $\overline\epsilon_{p} = \frac{1}{K}\sum_{k=1}^K \epsilon_{p, k}$.
\end{thm}

This theorem establishes a quantitative relationship between the privacy leakage measured with the total variation distance and the utility loss. It shows that the privacy disclosure can be confined by a function of the utility loss and the total variation distance across the safeguarded and unguarded model distributions. The implication is that reducing the utility loss and the difference across the protected and unprotected model distributions can help limit the privacy leakage.

Kindly consult \pref{sec:tvd PL} for a comprehensive analysis.

Both \pref{thm: utility-privacy trade-off_GJSD_mt} and \pref{thm: utility-privacy trade-off_TVD_mt} provide valuable insights into the inherent trade-off of privacy-utility within federated learning. They highlight the importance of designing protection mechanisms that can minimize the utility loss while maintaining a sufficient level of privacy protection. The theorems also suggest that the choice of privacy leakage measure ($\alpha$-skewed Jensen-Shannon divergence or total variation distance) can influence the quantitative relationship across privacy and utility.

\section{Trade-off Between Data Privacy and Model Utility}

\subsection{Measurements for Data Privacy and Model Utility}

We adopt the definition of privacy leakage from \cite{zhang2023game}, which quantifies the amount of private information exposed to the attacker. Let $\breve{d}^{(m)}$ denote the original $m$-th data sample, and $d_i^{(m)}$ represent the attacker's inferred $m$-th data sample at iteration $i$. The quantity of learning rounds is denoted by $I$. The privacy leakage $\epsilon_p$ is identified as:

\begin{equation}\label{eq: privacy_leakage}
\epsilon_p=\left\{
\begin{array}{cl}
\frac{D - \frac{1}{I}\sum_{i = 1}^{I} \frac{1}{|\mathcal{D}|}\sum_{m = 1}^{|\mathcal{D}|}\|d_i^{(m)} - \breve{d}^{(m)}\|}{D}, & I > 0\cr
0, & I = 0\cr
\end{array} \right.
\end{equation}
where $D$ is a positive constant that represents the maximum possible distance between the original and inferred data samples. Note that $\epsilon_p \in [0,1]$, with $\epsilon_p = 1$ indicating that the attacker has perfectly inferred the original data, and $\epsilon_p = 0$ indicating that no privacy leakage has occurred.



The utility loss measures the degradation in model performance due to the privacy-preserving mechanism. Let $f(\para;\cdot)$ denote the model parameterized by $\theta$, and signify the loss value with $\mathcal{L}(\cdot, \cdot)$. The utility loss $\epsilon_u$ is defined as:

\begin{equation}\label{eq: utility_loss}
\epsilon_u = \mathbb{E}_{(x, y) \sim \mathcal{D}} [\mathcal{L}(f(\para; \tilde{x} ), y) - \mathcal{L}(f(\para; x ), y)],
\end{equation}
where $\tilde{x}$ represents the distorted version of the input $x$ obtained by applying the privacy-preserving mechanism. The utility loss quantifies the difference in the expected loss between the original and distorted inputs, with $\epsilon_u = 0$ indicating no loss in utility, and larger values of $\epsilon_u$ indicating greater degradation in model performance.


The privacy leakage and utility loss metrics provide a quantitative way to measure the trade-off of privacy-utility within federated learning systems. The goal of privacy-preserving federated learning is to design mechanisms that minimize the privacy disclosure while keeping the utility cost within an tolerable range. In the following sections, we will introduce our learn-to-distort framework, which provides a principled approach for navigating the trade-off of privacy-utility by explicitly modeling the distortion as a learnable variable and optimizing it jointly with the model parameters.

\subsection{Theoretical Connections between Privacy-Utility Trade-off and \textit{Learn-to-Distort-Data} Problem}\label{subsec: theoretical_connections}

In this section, we establish the theoretical links between the trade-off of privacy-utility problem and the \textit{Learn-to-Distort-Data} problem. We first introduce the optimization problem for the trade-off of privacy-utility about federated learning and show how it can be reduced to the \textit{Learn-to-Distort-Data} problem. We then deduce a fundamental relationship between the privacy leakage and the distortion extent, which allows us to transform the constraint on privacy leakage into a constraint on distortion extent.

\subsubsection{The Trade-off of Privacy-utility Problem}

The trade-off of privacy-utility problem about federated learning can be formulated as the following limited optimisation issue:
\begin{align}
\begin{array}{rl}
\displaystyle \min_{\theta}&\calL(f(\theta; x + \delta), y) \\
\text{s.t.} &\epsilon_p\le\epsilon.
\end{array}
\end{align}

{Here, $x$ represents the input data, $\delta$ denotes the distortion applied to the data, $\para$ represents the model parameters, $f(\cdot)$ symbolizes the model function, $y$ means the matching label, and $\calL(\cdot)$ denotes the loss. The constraint $\epsilon_p\le\epsilon$ ensures that the privacy disclosure does not goes beyond a certain limit $\epsilon$ defined in advance.


In this optimization problem, $\delta$  is treated as an unlearnable variable. The objective of this optimization problem is to seek the ideal model parameters $\para$ that reduce the loss function while ensuring that the privacy leakage remains within an acceptable threshold. This formulation captures the inherent trade-off of privacy-utility in federated learning systems, where stronger privacy guarantees often come at the expense of reduced model utility.

\subsubsection{Reduction to \textit{Learn-to-Distort-Data} Problem}

To navigate the trade-off of privacy-utility, we propose to reduce the constrained optimization problem to a \textit{Learn-to-Distort-Data} problem. The key idea is to explicitly model the distortion $\delta$ as a learnable variable and optimize it jointly with the model parameters $\para$ to reach the ideal equilibrium across privacy and utility. The \textit{Learn-to-Distort-Data} problem can be formulated as follows:
\begin{align}
\begin{array}{r@{\quad}l@{}l@{\quad}l}
\quad\min_{\theta}&\min_{\delta}\calL(f(\para; x + \delta), y),\cr
\text{s.t.,} & \|\delta\|\ge\epsilon_1.
\end{array}
\end{align}

Here, $\epsilon_1$ is a predefined threshold that determines the minimum magnitude of the distortion required to provide the required level of privacy protection. The constraint $\|\delta\|\ge\epsilon_1$ ensures that the distortion is sufficiently large to achieve the desired level of privacy.

The optimization problem can be related to error-minimization
attacking for unlearnable example. In the \textit{Learn-to-Distort-Data} problem, the distortion is learnable and dynamically optimized, and the objective is to seek a harmonious equilibrium across safeguarding privacy and maximising the usefulness of the model while optimizing not only the privacy protection but also the model utility. In the error-minimization attacking for unlearnable example, the distortion is usually small and the objective of the attacker is to induce the model to provide incorrect forecasts.

To establish the connection between the trade-off of privacy-utility problem and the \textit{Learn-to-Distort-Data} problem, we need to deduce a connection between the privacy leakage $\epsilon_p$ and the distortion extent $\|\delta\|$.

\subsubsection{The Upper Threshold for Privacy Disclosure}

This section present a key theorem that establishes the correlation across the privacy leakage and the the level of distortion from the perspective of a semi-honest adversary in federated learning. The high-level thought of the proof is to first derive bounds on the cumulative regret of the adversary's optimization algorithm and the distance between data samples. Then, we use these bounds to establish a connection between the privacy leakage and the distortion extent. Finally, we prove the main theorem by leveraging the derived relationship.

\begin{thm}\label{thm: privacy_leakage_distortion_extent}
Consider a semi-honest adversary who employs an optimization strategy to conclude the client $k$'s initial data on the basis of the exposed model parameter $\theta$. Let $d$ denote the inferred data satisfying $g(d) = \theta$, and $\breve d$ represent the initial data satisfying $g(\breve d) = \breve \theta$, where $\breve \theta$ denotes the initial model parameter, $\theta$ denotes the protected model parameter, and $g(d) = \frac{\partial \mathcal{L} (\theta, d)}{\partial \theta}$. Let $d^{(m)}$ and $\breve d^{(m)}$ represent the $m$-th data sample of datasets $d$ and $\breve d$, severally. Define the distortion extent of the data as $\Delta = \left\|\frac{1}{n_k}\sum_{i = 1}^{n_k} d^{(i)} - \frac{1}{n_k}\sum_{i = 1}^{n_k} \breve d^{(i)}\right\|$. Assume the expected regret of the optimization algorithm over $I$ ($I > 0$) rounds is $\Theta(I^p)$.
If $\Delta\ge{2c_2 c_b}\cdot I^{p-1}$, then the following inequality holds:
\begin{align}
\epsilon_p \le 1 - \frac{\Delta + c_2\cdot c_b\cdot I^{p-1}}{4D}.
\end{align}
\end{thm}

This theorem provides an upper threshold on the privacy disclosure $\epsilon_p$ in terms of the distortion extent $\Delta$ and the number of iterations $I$ of the adversary's optimization algorithm. It shows that if the distortion extent is sufficiently large relative to the number of iterations, the privacy leakage can be effectively bounded. This result has important implications for the devise of privacy-preserving mechanisms in federated learning, as it suggests that introducing an appropriate level of distortion to the data can help mitigate the risk of privacy leakage.

The proof of \pref{thm: privacy_leakage_distortion_extent} leverages the correlation across the privacy leakage and the level of distortion established in \pref{lem: privacy_leakage_distortion_relationship}. By using the lower bound on $D(1-\epsilon_p)$ in terms of $\Delta$ and $I$, we can deduce an upper threshold on the privacy disclosure $\epsilon_p$. This result provides a theoretical foundation for the devise of privacy-preserving mechanisms in federated learning that aim to reach an appropriate degree of privacy protection by introducing an appropriate amount of distortion to the data.

\pref{thm: privacy_leakage_distortion_extent} states a fundamental link across the privacy leakage $\epsilon_p$ and the distortion extent $\Delta$. Intuitively, it shows the privacy leakage can be confined by a function of the distortion extent and the quantity of learning rounds. This relationship permits us to transform the constraint on privacy leakage into a constraint on distortion extent, which is more amenable to optimization. On the basis of \pref{thm: privacy_leakage_distortion_extent}, we are able to derive the next theorem, which establishes the equivalence between the trade-off of privacy-utility problem and the \textit{Learn-to-Distort-Data} problem:

\begin{thm}[Reduction of The Trade-off of Privacy-Utility to \textit{Learn-to-Distort-Data} Problem]\label{thm: reductionprivacy_utility_learn_to_distort}
Let $c = \frac{c_2\cdot c_b\cdot I^{p-1}}{4D}$, where $\epsilon_p$ represents the privacy leakage, $I$ represents the total number of learning rounds, and $c_2, c_b, D$ are constants. Define $\epsilon_1 = 4D \cdot (1 - c - \epsilon)$. The privacy-utility trade-off problem:
\begin{align}
\begin{array}{rl}
\displaystyle \min_{\theta}&\calL(f(\theta; x + \delta), y) \\
\text{s.t.} &\epsilon_p\le\epsilon.
\end{array}
\end{align}
can be reduced to the following \textit{Learn-to-Distort-Data} problem:

\begin{align}
\begin{array}{r@{\quad}l@{}l@{\quad}l}
\quad\min_{\theta}&\min_{\delta}\calL(f(\para; x + \delta), y),\cr
\text{s.t.,} & \|\delta\|\ge\epsilon_1.
\end{array}
\end{align}
\end{thm}

\begin{proof}
This theorem's proof leverages the relationship established in \pref{thm: privacy_leakage_distortion_extent}. By setting $\Delta\ge 4D \cdot (1 - c - \epsilon) = \epsilon_1$, we ensure that the privacy leakage restriction $\epsilon_p\le\epsilon$ is satisfied. Consequently, the constraint on the distortion extent $\|\delta\| = \|\distdata - \origdata\|\ge\epsilon_1$ in the \textit{Learn-to-Distort-Data} problem is equivalent to the confine on privacy disclosure $\epsilon_p\le\epsilon$ in the trade-off of privacy-utility problem.
\end{proof}

This theorem establishes a direct connection between the trade-off of privacy-utility problem and the \textit{Learn-to-Distort-Data} problem. By formulating the privacy-preserving federated learning problem as a \textit{Learn-to-Distort-Data} problem, we can leverage the techniques and insights from adversarial training, input-robustness, and unlearnable examples to design effective privacy-preserving mechanisms that reach the ideal equilibrium across privacy and utility.

\section{The \textit{Learn-to-Distort-Data} Framework for Privacy-Preserving Federated Learning}\label{sec: learn_to_distort}
The \textit{Learn-to-Distort-Data} framework provides a unified perspective on privacy-preserving mechanisms in federated learning that protect privacy by distorting the data. These mechanisms aim to optimize the model performance while ensuring that the privacy leakage remains within an acceptable threshold. The framework formulates the privacy-preserving federated learning problem as a constrained optimization problem:
\begin{align}
\begin{array}{r@{\quad}l@{}l@{\quad}l}
\quad\min_{\theta}&\min_{\delta}\calL(f(\theta; x + \delta), y),\cr
\text{s.t.,} & \|\delta\|\ge\epsilon_1.
\end{array}
\end{align}
Here, $\theta$ symbolizes the model parameters, $x$ denotes the input data, $\delta$ represents the distortion introduced by the privacy-preserving mechanism, $y$ means the matching label, and $\calL(\cdot)$ denotes the loss. The constraint $\|\delta\|\ge\epsilon_1$ ensures that the privacy disclosure does not goes beyond a certain limit $\epsilon$ defined in advance. The \textit{Learn-to-Distort-Data} framework captures a wide range of privacy-preserving mechanisms that distort the data to protect privacy in federated learning. By properly designing the distortion variable $\delta$ and the loss function $\calL(\cdot)$, the framework can optimize the model performance while maintaining the desired level of privacy protection. Here are some examples of privacy-preserving mechanisms that can be represented within the \textit{Learn-to-Distort-Data} framework:

\textbf{Differential Privacy (DP):} Differential privacy \cite{dwork2006calibrating, abadi2016deep} is a widely used privacy-preserving mechanism that introduces carefully calibrated noise to the data or model modifications to protect the confidentiality of personal data samples. Within the \textit{Learn-to-Distort-Data} framework, the distortion variable $\delta$ can be designed to represent the noise introduced by the DP mechanism. The optimization problem for DP-based federated learning can be formulated as:
\begin{align}
\begin{array}{r@{\quad}l@{}l@{\quad}l}
\quad\min_{\theta}&\min_{\sigma^2}\E_{\delta \sim \mathcal{N}(0, \sigma^2)}\calL(f(\theta; x + \delta), y),\cr
\text{s.t.,} & \sigma^2\in\calA.
\end{array}
\end{align}
Here, $\calA$ represents the set of valid noise variance values that satisfy the privacy constraint. The optimization problem aims to find the ideal model parameters $\theta$ and the noise variance $\sigma^2$ that minimize the expected loss function over the noise distribution while preserving the data's confidentiality.

\textbf{Secure Multiparty Computation (MPC):} Secure multiparty computation \cite{bonawitz2017practical} permits numerous clients to cooperatively calculate a function over their own data without disclosing the inputs to others. In the \textit{Learn-to-Distort-Data} framework, the distortion variable $\delta$ can be designed to represent the cryptographic primitives used in MPC, such as secret sharing or homomorphic encryption. The optimization problem for MPC-based federated learning can be formulated as:
\begin{align}
\begin{array}{r@{\quad}l@{}l@{\quad}l}
\quad\min_{\theta}&\min_{\delta}\calL(f(\theta; x + \delta), y),\cr
\text{s.t.,} & \delta \in \mathcal{C}.
\end{array}
\end{align}
Here, $\mathcal{C}$ denotes the set of valid cryptographic primitives that can be used in MPC. The optimization problem aims to find the ideal model parameters $\theta$ and the cryptographic primitives $\delta$ that minimize the loss function while ensuring the privacy of the data.

\textbf{Homomorphic Encryption (HE):} Homomorphic encryption enables computations to be carried out directly on data with noise, negating the demand for decryption. In the \textit{Learn-to-Distort-Data} framework, the distortion variable $\delta$ can be designed to represent the encryption scheme used in HE. The optimization problem for HE-based federated learning can be formulated as:
\begin{align}
\begin{array}{r@{\quad}l@{}l@{\quad}l}
\quad\min_{\theta}&\min_{\delta}\calL(f(\theta; x + \delta), y),\cr
\text{s.t.,} & \delta \in \mathcal{E}.
\end{array}
\end{align}
Here, $\mathcal{E}$ denotes the set of valid encryption schemes that can be used in HE. The optimization problem aims to find the ideal model parameters $\theta$ and the encryption scheme $\delta$ that minimize the loss function while preserving the privacy of the data.

\textbf{Data Compression:} Data compression strategies, like quantization \cite{reisizadeh2020fedpaq,shlezinger2020uveqfed} and sparsification \cite{han2020adaptive, panda2022sparsefed}, can be used to reduce the communication overhead and defend the confidentiality of the clients' data in federated learning. These techniques introduce distortion to the data by compressing or truncating model modifications prior to sharing them with the server. In the \textit{Learn-to-Distort-Data} framework, the distortion variable $\delta$ can be designed to represent the compression or truncation operation. The optimization problem for data compression-based federated learning can be formulated as:
\begin{align}
\begin{array}{r@{\quad}l@{}l@{\quad}l}
\quad\min_{\theta}&\min_{\delta}\calL(f(\theta; x + \delta), y),\cr
\text{s.t.,} & \delta\in\mathcal{T}.
\end{array}
\end{align}
The constraint $\delta\in\mathcal{T}$ ensures that the distortion is sufficiently large to provide the desired level of privacy protection. The optimization problem aims to find the best model parameters $\theta$ and the compression or truncation strategy $\delta$ that minimize the loss function while maintaining the desired level of privacy protection.
By unifying these privacy-preserving mechanisms under a common optimization problem, the \textit{Learn-to-Distort-Data} framework provides a principled way to study the trade-off of privacy-utility in federated learning. The framework allows for the advancement of optimal privacy-preserving mechanisms which reach a equilibrium across privacy protection and model performance, without compromising the privacy constraints. The flexibility and generality of the \textit{Learn-to-Distort-Data} framework make it a powerful tool for analyzing and developing privacy-preserving solutions in federated learning systems.

\section{Connections to Related Areas}\label{sec: related_areas}

In this section, we show the connections between the \textit{Learn-to-Distort-Data} formulation and related areas, and highlight how these connections can be leveraged to design effective algorithms for the trade-off of privacy-utility within federated learning.

\subsection{Connection to Adversarial Training}

Adversarial training (\cite{goodfellow2014explaining, madry2017towards}) is the technique that aims at enhancing the resilience of machine learning models against adversarial attacks. The central idea is to expand the training data with adversarially perturbed examples and train the model to accurately categorize these samples. The optimization problem in adversarial training is typically expressed as a min-max issue, where the inner maximization focuses on finding the perturbation that maximizes the loss, while the outer minimization seeks to update the model parameters to cut this adversarial loss.

The \textit{Learn-to-Distort-Data} formulation of privacy-preserving federated learning shares some similarities with adversarial training. In both cases, the objective is to seek the ideal perturbation or distortion which optimizes a certain objective function (e.g., minimizing the loss or maximizing the privacy) while satisfying certain constraints. As such, the techniques and insights from adversarial training, such as generating adversarial examples and optimizing model parameters to minimize the adversarial loss, can be adapted to generate privacy-preserving distortions and optimize the trade-off of privacy-utility within federated learning systems.

\subsection{Connection to Input-Robustness}

Input-robustness (\cite{zheng2016improving, kannan2018adversarial, yi2021improved}) refers to a model's ability to keep up its performance when noise or input disturbances are present. Input-robustness is a desirable property in many applications, such as image classification or speech recognition, in which the input could be subject to various types of distortions or corruptions.

The \textit{Learn-to-Distort-Data} formulation of privacy-preserving federated learning is also related to the concept of input-robustness. By explicitly modeling the distortion as a learnable variable and optimizing it to minimize the influence on the model's performance, the \textit{Learn-to-Distort-Data} formulation can be seen as a way to enforce input-robustness against the backdrop of privacy-preserving federated learning. The techniques and insights from input-robustness, such as data augmentation or stability training, can be adapted to design and analyze privacy-preserving mechanisms that are robust to input distortions.

\subsection{Connection to Unlearnable Examples}

Unlearnable examples (\cite{huang2021unlearnable, qin2023learning}) are input examples that are difficult or impossible for a model to learn from. Unlearnable examples can arise due to various factors, such as label noise, data corruption, or adversarial attacks.

Within the context of privacy-preserving federated learning, the distortion introduced to safeguard privacy can be seen as a form of data corruption that may render some examples unlearnable. The \textit{Learn-to-Distort-Data} formulation provides a way to explicitly model and optimize the distortion to minimize the impact on the model's ability to study from the distorted examples. The techniques and insights from the study of unlearnable examples, such as curriculum learning or robust optimization, can be adapted to design and analyze privacy-preserving mechanisms that are resilient to the presence of unlearnable examples.

By establishing these connections to related areas, we highlight the potential of the \textit{Learn-to-Distort-Data} framework in leveraging the techniques and insights from these areas to design effective algorithms for privacy-utility trade-off in federated learning. These connections also open up new possibilities for future research and development in the area of privacy-preserving machine learning.

\section{Conclusion}\label{sec: conclusion}

This paper proposed a unified \textit{Learn-to-Distort-Data} framework for privacy-preserving federated learning, which formulates the issue as a constrained optimisation issue and reduces it to a \textit{Learn-to-Distort-Data} problem. We demonstrated that the \textit{Learn-to-Distort-Data} framework can accommodate diverse privacy-preserving mechanisms used in federated learning, including differential privacy, secure multiparty computation, and homomorphic encryption. We highlighted that the \textit{Learn-to-Distort-Data} formulation implies connections to various related areas, such as adversarial training, input-robustness, and unlearnable examples, which can be leveraged to design effective algorithms for the trade-off of privacy-utility in federated learning. We provided a rigorous theoretical analysis of the trade-off of privacy-utility under our \textit{Learn-to-Distort-Data} framework, which offers valuable insights for the design and evaluation of privacy-preserving mechanisms for federated learning environments.

Our work makes several important contributions to the realm of privacy-preserving federated learning. First, by formulating the problem as a \textit{Learn-to-Distort-Data} problem, we provide a unified framework that can accommodate a wide range of privacy-preserving mechanisms and enable the design of optimal mechanisms that reach a equilibrium across privacy protection and model performance. Second, by establishing connections to related areas, we highlight the potential of leveraging techniques and insights from these areas to design effective algorithms for the trade-off of privacy-utility in federated learning. Third, our theoretical analysis provides a rigorous foundation for understanding the trade-off of privacy-utility in federated learning and offers valuable insights for the creation and evaluation of privacy-preserving mechanisms.

There are several promising directions for future research. One direction is to explore more sophisticated distortion mechanisms that can generate more realistic and informative distortions, such as generative models or autoencoder-based techniques. Another direction is to develop more efficient algorithms for solving the \textit{Learn-to-Distort-Data} optimization problem, which can scale to large-scale federated learning systems with millions of clients and high-dimensional data. The last direction is to investigate the robustness and stability of the \textit{Learn-to-Distort-Data} framework under different types of attacks and noise models, and develop techniques for mitigating the influence of these considerations on the privacy and utility of the learned model.

In conclusion, we believe that our \textit{Learn-to-Distort-Data} framework provides a principled and flexible approach for navigating the trade-off of privacy-utility in trustworthy federated learning systems. Our theoretical analysis and algorithmic design offer valuable insights and tools for the development of privacy-preserving federated learning systems that can enable secure and effective collaboration among numerous clients while defending the confidentiality of individual participants. We desire that our work will stimulate further research and development in this important area, and contribute to the realization of the whole potential of federated learning in real-world applications.

\bibliographystyle{plain}
\bibliography{main}

\newpage

\begin{appendix}

\section{Notations used in the paper}\label{sec: notations table}

\begin{table}[h]
\small
\centering
\caption{Notations used in the paper}
\begin{tabular}{|p{0.2\textwidth}|p{0.8\textwidth}|}
\hline
\centering \textbf{Notation} & \centering \textbf{Explain}  \cr
\hline
\centering $D_k$ & \centering the private data of client $k$ \cr
\hline
\centering $\calD_k$ & \centering the support set of $D_k$ \cr
\hline
\centering $F^{\calA}_k$ & \centering the client's private data $D_k$ upon observing the protected information \cr
\hline
\centering $F^{\calO}_k$ & \centering the client's private data $D_k$ upon without observing any information \cr
\hline
\centering $F^{*}_k$ & \centering the client's private data $D_k$ upon observing the unprotected information \cr
\hline
\centering $w$ & \centering the model parameters \cr
\hline
\centering $\epsilon_{u}$ & \centering the utility loss \cr
\hline
\centering $U$ & \centering the model utility \cr
\hline
\centering $\mathcal W^*_k$ & \centering  the utility resulted from sharing any parameter in $\mathcal W^*_k$ achieves the maximal value. \cr
\hline
\centering $\Delta_k$ & \centering the maximum utility difference between the protected model parameters and the optimal model parameters for client $k$ \cr
\hline
\centering $\epsilon_{p, k}$ & \centering the privacy leakage of client $k$ \cr
\hline
\centering $\epsilon_{p}$ & \centering the average privacy leakage \cr
\hline
\centering $JS_{\alpha}$ & \centering the generalized Jensen-Shannon divergence \cr
\hline
\centering $JS$ & \centering the standard Jensen-Shannon divergence \cr
\hline
\centering $x$ & \centering the input data or training data \cr
\hline
\centering $y$ & \centering the corresponding label  \cr
\hline
\centering $T$ & \centering the distribution of the training data set  \cr
\hline
\centering $\calL(\cdot)$ & \centering the loss function  \cr
\hline
\centering $f(\cdot)$ & \centering the model function  \cr
\hline
\centering $\delta$ & \centering the distortion  \cr
\hline
\centering $\para$ & \centering the model parameters  \cr
\hline
\centering $P^{o}_a$ & \centering the distribution of the aggregated parameter before being protected  \cr
\hline
\centering $P^d_a$ & \centering the distribution of the aggregated parameter after being protected  \cr
\hline
\centering $\calW^o_k$ & \centering the support set of the $P^{o}_a$  \cr
\hline
\centering $\calW^d_k$ & \centering the support set of the $P^d_a$  \cr
\hline
\end{tabular}
\end{table}

\section{Theoretical Analysis for Privacy Leakage}

\begin{thm}\label{thm: privacy_leakage_distortion_extent_app}
Consider a semi-honest adversary who employs an optimization algorithm to infer the original data of client $k$ based on the exposed model parameter $\theta$. Let $d$ represent the inferred data satisfying $g(d) = \theta$, and $\breve d$ represent the original data satisfying $g(\breve d) = \breve \theta$, where $\breve \theta$ denotes the original model parameter, $\theta$ denotes the protected model parameter, and $g(d) = \frac{\partial \mathcal{L} (\theta, d)}{\partial \theta}$. Let $d^{(m)}$ and $\breve d^{(m)}$ represent the $m$-th data sample of datasets $d$ and $\breve d$, respectively. Define the distortion extent of the data as $\Delta = \left\|\frac{1}{n_k}\sum_{i = 1}^{n_k} d^{(i)} - \frac{1}{n_k}\sum_{i = 1}^{n_k} \breve d^{(i)}\right\|$. Assume the expected regret of the optimization algorithm over $I$ ($I > 0$) rounds is $\Theta(I^p)$.
If $\Delta\ge{2c_2 c_b}\cdot I^{p-1}$, then the following inequality holds:
\begin{align}
\epsilon_p \le 1 - \frac{\Delta + c_2\cdot c_b\cdot I^{p-1}}{4D}.
\end{align}
\end{thm}

To prove this theorem, we first introduce several lemmas that establish the relationship between the privacy leakage and the distortion extent.

\begin{lem}[Regret Bounds]\label{lem: regret_bounds}
Let $d_i$ denote the remodeled data at iteration $i$ using the optimization algorithm. The cumulative regret over $I$ rounds satisfies the following bounds:
\begin{align}\label{eq: regret_bounds}
c_1\cdot I^p \le \sum_{i = 1}^{I}\left\|g(d_i) - g(d)\right\| = \Theta(I^p) \le c_2\cdot I^p,
\end{align}
where $c_1$ and $c_2$ are constants independent of $I$.
\end{lem}

This lemma provides bounds on the cumulative regret of the adversary's optimization algorithm over $I$ rounds. It shows that the regret grows at a rate of $\Theta(I^p)$, where $p$ is a constant that depends on the specific algorithm used. These bounds will be used later in the proof of the main theorem.

\begin{proof}
The cumulative regret over $I$ rounds is defined as:
\begin{align*}
R(I) & = \sum_{i = 1}^{I} \left[\left\|g(d_i) - \theta\right\| - \left\|g(d) - \theta\right\|\right]\\
& = \sum_{i = 1}^{I} \left[\left\|g(d_i) - g(d)\right\|\right]\\
& = \Theta(I^p).
\end{align*}
Therefore, we have:
\begin{align*}
c_1\cdot I^p \le \sum_{i = 1}^{I}\left\|g(d_i) - g(d)\right\| = \Theta(I^p) \le c_2\cdot I^p,
\end{align*}
where $c_1$ and $c_2$ are constants independent of $I$.
\end{proof}

\begin{lem}[Data Distance Bounds]\label{lem: data_distance_bounds}
Let $x$ and $\tilde x$ represent two data samples. We assume that the following bounds hold:
\begin{align}
c_a \left\|g(x) - g(\tilde x)\right\|\le \left\|x - \tilde x\right\|\le c_b \left\|g(x) - g(\tilde x)\right\|,
\end{align}
where $c_a$ and $c_b$ are constants independent of $x$ and $\tilde x$.
\end{lem}

This lemma assumes that the distance between two data samples can be bounded in terms of the distance between their corresponding model parameters under the mapping $g$. These bounds will be used to establish the relationship between the privacy leakage and the distortion extent.

\begin{lem}[Privacy Leakage and Distortion Relationship]\label{lem: privacy_leakage_distortion_relationship}
Let $d_i^{(m)}$ denote the remodeled $m$-th data sample at iteration $i$ using the optimization algorithm. We have the following relationship between the privacy leakage and the distortion extent:
\begin{align*}
D(1-\epsilon_p) = \frac{1}{I}\sum_{i = 1}^{I} \frac{1}{n_k}\sum_{m = 1}^{n_k}\left\|d_i^{(m)} - \breve d^{(m)}\right\|
&\ge \Delta - c_b\cdot\frac{1}{I}\sum_{i = 1}^{I}\sum_{m = 1}^{n_k}\left\|g(d_i^{(m)}) - g(d^{(m)})\right\|\\
&\ge \Delta - c_2\cdot c_b I^{p-1}\\
&\ge\frac{1}{2}\max\left\{\Delta, c_2\cdot c_b I^{p-1}\right\} \\
&\ge \frac{\Delta + c_2\cdot c_b I^{p-1}}{4}.
\end{align*}
\end{lem}

This lemma establishes a key relationship between the privacy leakage $\epsilon_p$ and the distortion extent $\Delta$. It shows that the privacy leakage can be bounded in terms of the distortion extent and the cumulative regret of the adversary's optimization algorithm. This relationship forms the basis for proving the main theorem.

\begin{proof}
Recall the privacy leakage of client $k$ is defined as:
\begin{equation}
\epsilon_p=\left\{
\begin{array}{cl}
\frac{D - \frac{1}{I}\sum_{i = 1}^{I} \frac{1}{n_k}\sum_{m = 1}^{n_k}\left\|d_i^{(m)} - \breve d^{(m)}\right\|}{D}, & I > 0\\
0, & I = 0\\
\end{array} \right.
\end{equation}

We have:
\begin{align*}
\frac{1}{n_k}\sum_{m = 1}^{n_k}\left\|d_i^{(m)} - \breve d^{(m)}\right\|
&\ge \frac{1}{n_k}\sum_{m = 1}^{n_k}\left\|d^{(m)} - \breve d^{(m)}\right\| - \frac{1}{n_k}\sum_{m = 1}^{n_k}\left\|d_i^{(m)} - d^{(m)}\right\|\\
& \ge \Delta - c_b \sum_{m = 1}^{n_k}\left\|g(d_i^{(m)}) - g(d^{(m)})\right\|,
\end{align*}
where the second inequality is due to \pref{lem: data_distance_bounds}.

Therefore, we have:
\begin{align*}
D(1-\epsilon_p) = \frac{1}{I}\sum_{i = 1}^{I} \frac{1}{n_k}\sum_{m = 1}^{n_k}\left\|d_i^{(m)} - \breve d^{(m)}\right\|
&\ge \Delta - c_b\cdot\frac{1}{I}\sum_{i = 1}^{I}\sum_{m = 1}^{n_k}\left\|g(d_i^{(m)}) - g(d^{(m)})\right\|\\
&\ge \Delta - c_2\cdot c_b I^{p-1} \quad \text{(by \pref{lem: regret_bounds})}\\
&\ge\frac{1}{2}\max\left\{\Delta, c_2\cdot c_b I^{p-1}\right\}\\
&\ge \frac{\Delta + c_2\cdot c_b I^{p-1}}{4}.
\end{align*}
\end{proof}
Proof of \pref{thm: privacy_leakage_distortion_extent}:
\begin{proof}
From \pref{lem: privacy_leakage_distortion_relationship}, we have:
\begin{align*}
D(1-\epsilon_p) \ge \frac{\Delta + c_2\cdot c_b I^{p-1}}{4}.
\end{align*}
Therefore, we have:
\begin{align*}
\epsilon_{p} &\le 1 - \frac{\Delta + c_2\cdot c_b I^{p-1}}{4D}.
\end{align*}
This completes the proof of \pref{thm: privacy_leakage_distortion_extent}.
\end{proof}

\section{Theoretical Analysis for Utility Loss}
\label{sec:BP}

\begin{lem}[\cite{duchi2013local}]\label{lem: log_upper_bound}
For two positive numbers $a$ and $b$, we have
\begin{align*}
 \left|\log\left(\frac{a}{b}\right)\right| \le \frac{|a-b|}{\min\{a, b\}}.
\end{align*}
\end{lem}

\subsection{The quantitative relationship between $||P^d - P^{o}||_{\text{TV}}$ and $\epsilon_u$}

\begin{lem}\label{lem: total_variation-utility trade-off}
Let \pref{assump: defi_of_Delta} hold, and $\epsilon_{u}$ be defined in \pref{defi: utility_loss}. Let $P^{o}_a$ and $P^d_a$ represent the distribution of the aggregated parameter before and after being protected. Then we have,
\begin{align*}
    \epsilon_{u} \ge \frac{1}{2K}\sum_{k=1}^K \Delta_k||P^d_a - P^{o}_a||_{\text{TV}}.
\end{align*}
\end{lem}

\begin{proof}

Let $\mathcal U_k = \{w\in\mathcal W_k: dP^{d}(w) - dP^{o}(w)>0\}$, and $\mathcal V_k = \{w\in\mathcal W_k: dP^{d}(w) - dP^{o}(w)\le 0\}$, $\calW_{\Delta_k} = \{w\in\calW^d_k: 0<|U(w, D_k) - U(w^{\star}, D_k)|\le\Delta_k\}$. For any $w\in\mathcal V_k$, the definition of $\mathcal V_k$ implies that $dP^{o}(w)\ge dP^{d}(w)\ge 0$. Therefore, w belongs to the support of $P^{o}$, which is denoted as $\calW^o_k$. Then, we have
\begin{align*}
\mathcal V_k\subset\calW^o_k\subset\mathcal W^{\star}_k
\end{align*}

Similarly, we have that
\begin{align*}
\mathcal U_k\subset\calW^d_k
\end{align*}

We decompose $\int_{\mathcal{U}_k} U(w, d)\one_{w\in\calW_{\Delta_k}}[d P^{d}_a(w) - d P^{o}_a(w)]$ as the summation of $\int_{\mathcal{U}_k} U(w, d)\one_{w\in\calW_{\Delta_k}}[d P^{d}_a(w) - d P^{o}_a(w)]$ and $\int_{\mathcal{U}_k} U(w, d)\one_{w\notin\calW_{\Delta_k}}[d P^{d}_a(w) - d P^{o}_a(w)]$.

Then we have
\begin{align*}
&\epsilon_{u} = \frac{1}{K}\sum_{k=1}^K \epsilon_{u, k}\\
    & = \frac{1}{K}\sum_{k=1}^K [U(P^{o}_a, D_k) - U(P^d_a, D_k)]\\  
    &= \frac{1}{n_k}\sum_{d\in\calD_k} \mathbb E_{w\sim P^{o}_a}U(w, d) - \frac{1}{n_k}\sum_{d\in\calD_k} \mathbb E_{w\sim P^d_a}U(w, d)\\
    & = \frac{1}{K}\sum_{k=1}^K\left[\frac{1}{n_k}\sum_{d\in\calD_k} \int_{\mathcal W_k} U(w, d) dP^{o}(w) - \frac{1}{n_k}\sum_{d\in\calD_k} \int_{\mathcal W_k} U(w, d)dP^{d}(w)\right]\\
    & = \frac{1}{K}\sum_{k=1}^K\left[\frac{1}{n_k}\sum_{d\in\calD_k}\int_{\mathcal{V}_k} U(w, d)[d P^{o}_a(w) - d P^{d}_a(w)] - \frac{1}{n_k}\sum_{d\in\calD_k}\int_{\mathcal{U}_k} U(w, d)[d P^{d}_a(w) - d P^{o}_a(w)]\right]\\
    &\ge\frac{1}{K}\sum_{k=1}^K\left[\Delta_k\cdot[||P^d_a - P^{o}_a||_{\text{TV}} - \frac{1}{n_k}\sum_{d\in\calD_k}\int_{\mathcal{U}_k} U(w, d)\one_{w\in\calW_{\Delta_k}}[d P^{d}_a(w) - d P^{o}_a(w)]\right]\\
    &\ge\frac{1}{K}\sum_{k=1}^K\Delta_k\cdot[||P^d_a - P^{o}_a||_{\text{TV}} - \int_{\calW_{\Delta_k}} p^{d}_{W_k}(w) dw]\\
    &\ge\frac{1}{K}\sum_{k=1}^K\Delta_k\cdot\frac{||P^d_a - P^{o}_a||_{\text{TV}}}{2},
\end{align*}
where the first inequality is due to $\mathcal V_k\subset \mathcal W^{\star}_k$, the second inequality is due to $\int_{\mathcal{U}_k} U(w, d)\one_{w\in\calW_{\Delta_k}}[d P^{d}_a(w) - d P^{o}_a(w)] \le \int_{\mathcal{U}_k} U(w, d)\one_{w\in\calW_{\Delta_k}}d P^{d}_a(w) \le \int_{\mathcal{W}^d_k} U(w, d)\one_{w\in\calW_{\Delta_k}}d P^{d}_a(w) = \int_{\calW_{\Delta_k}} p^{d}_{W_k}(w) dw $, the third inequality is due to $\int_{\calW_{\Delta_k}} p^{d}_{W_k}(w) dw\le\frac{||P^d_a - P^{o}_a||_{\text{TV}}}{2}$.
\end{proof}

\section{The Theoretical Analysis using Generalized JS Divergence}\label{sec: generalized_JSD}

\subsection{The quantitative relationship between $||P^d_k - P^{o}_k||_{\text{TV}}$ and $\epsilon_{p, k}$}

\begin{lem}\label{lem: GeneralizedJSBound_app}
Define the privacy leakage of any parameter $w$ released by client $k$ as $\epsilon_{p,w,k} = \log\left(\frac{f_{D_k|W_k}(d|w)}{f_{D_k}(d)}\right)$. Let \pref{assump: defi_of_Delta} hold, and $\epsilon_{p,k}$ be defined in \pref{defi: average_privacy_dtv_app}. Let $P^{o}_k$ and $P^d_k$ represent the distribution of the parameter of client $k$ before and after being protected. Let $F^{\calA}_k$ and $F^{\star}_k$ represent the belief of client $k$ about $D$ after observing the protected and original parameter. Let $F^{\calM_\alpha}_k$ represent a smoothed distribution, the density function of which is defined as $f^{\calM_\alpha}_{D_k}(d) = \alpha f^{\calA}_{D_k}(d) + (1-\alpha)f^{\star}_{D_k}(d)$. Define $JS_{\alpha}(F^{\calA}_k || F^{\star}_k) = \alpha KL(F^{\calA}_k || F^{\calM_\alpha}_k) + (1-\alpha)KL(F^{\star}_k || F^{\calM_\alpha}_k)$.
Then, we have
\begin{align*}
JS_{\alpha}(F^{\calA}_k || F^{\star}_k)\le 2\alpha(1-\alpha)(e^{2\delta}-1)||P^{o}_k - P^{d}_k||_{\text{TV}}. 
\end{align*}
\end{lem}

\begin{proof}

The generalized JS metric satisfies that 
\begin{align*}
    JS_{\alpha}(F^{\calA}_k || F^{\star}_k)\le \max\{2\alpha KL(F^{\calA}_k || F^{\calM_\alpha}_k), 2(1-\alpha)KL(F^{\star}_k || F^{\calM_\alpha}_k)\}.
\end{align*}

Notice that
\begin{align*}
    KL(F^{\calA}_k || F^{\calM_\alpha}_k) = \int_{\mathcal{\calD}_k} f^{\calA}_{D_k}(d)\log\frac{f^{\calA}_{D_k}(d)}{f^{\calM_\alpha}_{D_k}(d)}d\mu(s), 
\end{align*}

and 

\begin{align*}
    KL(F^{\star}_k || F^{\calM_\alpha}_k) = \int_{\mathcal{\calD}_k} f^{\star}_{D_k}(d)\log\frac{f^{\star}_{D_k}(d)}{f^{\calM_\alpha}_{D_k}(d)}d\mu(d). 
\end{align*}

Therefore,
\begin{align*}
    JS_{\alpha}(F^{\calA}_k || F^{\star}_k)&\le \max\left\{2\alpha\int_{\mathcal{\calD}_k} f^{\calA}_{D_k}(d)\left|\log\frac{f^{\calA}_{D_k}(d)}{f^{\calM_\alpha}_{D_k}(d)}\right|d\mu(d), 2(1-\alpha)\int_{\mathcal{\calD}_k} f^{\star}_{D_k}(d)\left|\log\frac{f^{\star}_{D_k}(d)}{f^{\calM_\alpha}_{D_k}(d)}\right|d\mu(d)\right\}\\
    &\le \max\left\{2\alpha\left|\log\frac{f^{\calA}_{D_k}(d)}{f^{\calM_\alpha}_{D_k}(d)}\right|, 2(1-\alpha)\left|\log\frac{f^{\star}_{D_k}(d)}{f^{\calM_\alpha}_{D_k}(d)}\right|\right\}.
\end{align*}

\textbf{Bounding $\left|f^{\calA}_{D_k}(d) - f^{\star}_{D_k}(d)\right|$.}

Let $\mathcal U_k = \{w\in\mathcal W_k: dP^{d}_k(w) - dP^{o}_k(w)>0\}$, and $\mathcal V_k = \{w\in\mathcal W_k: dP^{d}_k(w) - dP^{o}_k(w)\le 0\}$. Then we have 

\begin{align*}
    \left|f^{\calA}_{D_k}(d) - f^{\star}_{D_k}(d)\right| &= \left|\int_{\mathcal W_k} f_{D_k|W_k}(d|w)[d P^{d}_k(w) - d P^{o}_k(w)]\right|\nonumber\\
    &= \left|\int_{\mathcal{U}_k} f_{D_k|W_k}(d|w)[d P^{d}_k(w) - d P^{o}_k(w)] + \int_{\mathcal{V}_k} f_{D_k|W_k}(d|w)[d P^{d}_k(w) - d P^{o}_k(w)]\right|\nonumber\\
    &\le\sup_{w\in\mathcal{W}_k} f_{D_k|W_k}(d|w)\int_{\U} [dP^{d}_k(w) - dP^{o}_k(w)] + \inf_{w\in\mathcal{W}_k} f_{D_k|W_k}(d|w)\int_\V[dP^{d}_k(w) - dP^{o}_k(w)]\nonumber\\
    &=\left(\sup_{w\in\mathcal{W}_k} f_{D_k|W_k}(d|w) - \inf_{w\in\mathcal{W}_k} f_{D_k|W_k}(d|w)\right)\int_{\mathcal{U}_k} [d P^{d}_k(w) - d P^{o}_k(w)].
\end{align*}

First, we provide an upper bound for $(\sup_{w\in\mathcal{W}_k} f_{D_k|W_k}(d|w) - \inf_{w\in\mathcal{W}_k} f_{D_k|W_k}(d|w))$.

For any pair $w$ and $w'$, we have
\begin{align*}
    |f_{D_k|W_k}(d|w) - f_{D_k|W_k}(d|w')| = f_{D_k|W_k}(d|w')\left|\frac{f_{D_k|W_k}(d|w)}{f_{D_k|W_k}(d|w')}-1\right|.
\end{align*}

Therefore,

\begin{align*}
    \sup_{w\in\calW_k} f_{D_k|W_k}(d|w) - \inf_{w\in\calW_k} f_{D_k|W_k}(d|w) = \inf_{w\in\calW_k} f_{D_k|W_k}(d|w)\left|\frac{\sup_{w\in\calW_k} f_{D_k|W_k}(d|w)}{\inf_{w\in\calW_k} f_{D_k|W_k}(d|w)}-1\right|.
\end{align*}

From the definition of maximum privacy leakage, we know that
\begin{align*}
    \frac{f_{D_k|W_k}(d|w)}{f_{D_k}(d)}\in [e^{-\delta}, e^{\delta}],
\end{align*}
for any $w\in\mathcal W_k$.

Therefore, for any $w, w'\in\mathcal W_k$ we have
\begin{align*}
    \frac{f_{D_k|W_k}(d|w)}{f_{D_k|W_k}(d|w')} = \frac{f_{D_k|W_k}(d|w)}{f_{D_k}(d)}/\frac{f_{D_k|W_k}(d|w')}{f_{D_k}(d)}\in [e^{-2\delta}, e^{2\delta}]. 
\end{align*}

Therefore, the first term of \pref{eq:initial_step} is bounded by 
\begin{align}\label{eq: bound_1_term_1}
    \sup_{w\in\calW_k} f_{D_k|W_k}(d|w) - \inf_{w\in\calW_k} f_{D_k|W_k}(d|w)\le \inf_{w} f_{D_k|W_k}(d|w)(e^{2\delta}-1).
\end{align}

From the definition of total variation distance, we have
\begin{align}\label{eq: bound_1_term_2}
    \int_\U [dP^{d}_k(w) - dP^{o}_k(w)(w)] = ||P^{o}_k - P^{d}_k||_{\text{TV}}.
\end{align}

Combining \pref{eq: bound_1_term_1} and \pref{eq: bound_1_term_2}, we have
\begin{align}\label{eq: bound_for_the_gap_general}
        |f^{\calA}_{D_k}(d) - f^{\star}_{D_k}(d)| &=\left(\sup_{w\in\calW_k} f_{D_k|W_k}(d|w) - \inf_{w\in\calW_k} f_{D_k|W_k}(d|w)\right)\int_\U [dP^{d}_k(w) - dP^{o}_k(w)]\nonumber\\
        &\le \inf_{w\in\calW_k} f_{D_k|W_k}(d|w)(e^{2\delta}-1)||P^{o}_k - P^{d}_k||_{\text{TV}}.
\end{align}

\textbf{Bounding $\log\frac{f^{\calA}_{D_k}(d)}{f^{\calM_\alpha}_{D_k}(d)}$ and $\log\frac{f^{\star}_{D_k}(d)}{f^{\calM_\alpha}_{D_k}(d)}$.}

From \pref{lem: log_upper_bound}, we have

\begin{align*}
    \left|\log\frac{f^{\calA}_{D_k}(d)}{f^{\calM_\alpha}_{D_k}(d)}\right|&\le\frac{|f^{\calA}_{D_k}(d) - f^{\calM_\alpha}_{D_k}(d)|}{\min\{f^{\calA}_{D_k}(d), f^{\calM_\alpha}_{D_k}(d)\}},
\end{align*}

and

\begin{align*}
    \left|\log\frac{f^{\star}_{D_k}(d)}{f^{\calM_\alpha}_{D_k}(d)}\right|&\le\frac{|f^{\calA}_{D_k}(d) - f^{\calM_\alpha}_{D_k}(d)|}{\min\{f^{\calA}_{D_k}(d), f^{\calM_\alpha}_{D_k}(d)\}}.
\end{align*}

Recall that $f^{\calM_\alpha}_{D_k}(d) = \alpha f^{\calA}_{D_k}(d) + (1-\alpha)f^{\star}_{D_k}(d)$. Therefore, we have that
\begin{align*}
    f^{\calA}_{D_k}(d) - f^{\calM_\alpha}_{D_k}(d)  &= (1-\alpha)(f^{\calA}_{D_k}(d) - f^{\star}_{D_k}(d)),
\end{align*}

and

\begin{align*}
    f^{\star}_{D_k}(d) - f^{\calM_\alpha}_{D_k}(d)  &= \alpha (f^{\star}_{D_k}(d) - f^{\calA}_{D_k}(d)).
\end{align*}

We also have that
\begin{align*}
    \min\{f^{\calA}_{D_k}(d), f^{\calM_\alpha}_{D_k}(d)\}\ge\min\{f^{\calA}_{D_k}(d), f^{\star}_{D_k}(d)\},
\end{align*}

and
\begin{align*}
    \min\{f^{\star}_{D_k}(d), f^{\calM_\alpha}_{D_k}(d)\}\ge\min\{f^{\star}_{D_k}(d), f^{\calA}_{D_k}(d)\}.
\end{align*}

Therefore, 
\begin{align*}
    \max_{d\in\mathcal D_k}&\left\{2\alpha\left|\log\frac{f^{\calA}_{D_k}(d)}{f^{\calM_\alpha}_{D_k}(d)}\right|, 2(1-\alpha)\left|\log\frac{f^{\star}_{D_k}(d)}{f^{\calM_\alpha}_{D_k}(d)}\right|\right\}\nonumber\\
    &\le\frac{\max\{2\alpha|f^{\calA}_{D_k}(d) - f^{\calM_\alpha}_{D_k}(d)|, 2(1-\alpha)|f^{\star}_{D_k}(d) - f^{\calM_\alpha}_{D_k}(d)|\}}{\min\{f^{\calA}_{D_k}(d), f^{\calM_\alpha}_{D_k}(d)\}}\nonumber\\
    &\le \frac{2\alpha(1-\alpha)|f^{\calA}_{D_k}(d) - f^{\star}_{D_k}(d)|}{\min\{f^{\calA}_{D_k}(d), f^{\star}_{D_k}(d)\}}\nonumber\\
    &\le 2\alpha(1-\alpha)(e^{2\delta}-1)||P^{o}_k - P^{d}_k||_{\text{TV}},
\end{align*}
where the third inequality is due to $\min\{f^{\calA}_{D_k}(d), f^{\star}_{D_k}(d)\}\ge \inf_{w} f_{D_k|W_k}(d|w)$ and \pref{eq: bound_for_the_gap_general}.

Therefore, 
\begin{align*}
   JS_{\alpha}(F^{\calA}_k || F^{\star}_k)&\le \max_{d\in\calD_k}\left\{2\alpha\left|\log\frac{f^{\calA}_{D_k}(d)}{f^{\calM_\alpha}_{D_k}(d)}\right|, 2(1-\alpha)\left|\log\frac{f^{\star}_{D_k}(d)}{f^{\calM_\alpha}_{D_k}(d)}\right|\right\}\\
    &\le 2\alpha(1-\alpha)(e^{2\delta}-1)||P^{o}_k - P^{d}_k||_{\text{TV}}.
\end{align*}

\end{proof}


\begin{lem}\label{lem: total_variation-privacy trade-off_GJSD}
Define the privacy leakage of any parameter $w$ released by client $k$ as $\epsilon_{p,w,k} = \log\left(\frac{f_{D_k|W_k}(d|w)}{f_{D_k}(d)}\right)$. Let \pref{assump: defi_of_Delta} hold, and $\epsilon_{p,k}$ be defined in \pref{defi: average_privacy_dtv_app}. Let $P^{o}_k$ and $P^d_k$ represent the distribution of the parameter of client $k$ before and after being protected. Let $F^{\calB}_k$ and $F^{\calA}_k$ represent the belief of client $k$ about $D$ before and after observing the released parameter. Then, we have
\begin{align*}
    JS_{\alpha}(F^{\star}_k || F^{\calO}_k)^{1/e}\le\epsilon_{p,k} + \left[2\alpha(1-\alpha)(e^{2\delta}-1)||P^{o}_k - P^{d}_k||_{\text{TV}}\right]^{1/e}.
\end{align*}
\end{lem}

\begin{proof}
From \pref{lem: GeneralizedJSBound_app}, we know that 
\begin{align*}
JS_{\alpha}(F^{\calA}_k || F^{\star}_k)\le 2\alpha(1-\alpha)(e^{2\delta}-1)||P^{o}_k - P^{d}_k||_{\text{TV}}. 
\end{align*}

Notice that $JS_{\alpha}^{1/e}$ satisfies the triangle inequality from \pref{lem: triangle_inequality_mt}. That is, $JS_{\alpha}(F^{\star}_k || F^{\calO}_k)^{1/e}\le JS_{\alpha}(F^{\calA}_k, F^{\calO}_k)^{1/e} + JS_{\alpha}(F^{\calA}_k || F^{\star}_k)^{1/e}$.

\begin{align*}
    JS_{\alpha}(F^{\star}_k || F^{\calO}_k)^{1/e} - JS_{\alpha}(F^{\calA}_k || F^{\calO}_k)^{1/e}\le JS_{\alpha}(F^{\calA}_k || F^{\star}_k)^{1/e}\le \left[2\alpha(1-\alpha)(e^{2\delta}-1)||P^{o}_k - P^{d}_k||_{\text{TV}}\right]^{1/e}.
\end{align*}

Therefore, we have that
\begin{align*}
    JS_{\alpha}(F^{\star}_k || F^{\calO}_k)^{1/e} &\le JS(F^{\calA}_k || F^{\calO}_k)^{1/e} + \left[2\alpha(1-\alpha)(e^{2\delta}-1)||P^{o}_k - P^{d}_k||_{\text{TV}}\right]^{1/e}\\
    & = \epsilon_{p,k} + \left[2\alpha(1-\alpha)(e^{2\delta}-1)||P^{o}_k - P^{d}_k||_{\text{TV}}\right]^{1/e}.
\end{align*}
\end{proof}

\subsection{Analysis of \pref{thm: utility-privacy trade-off_GJSD_mt}}

With \pref{lem: total_variation-privacy trade-off_GJSD} and \pref{lem: total_variation-utility trade-off}, it is now natural to provide a quantitative relationship between the utility loss and the privacy leakage (\pref{thm: utility-privacy trade-off_GJSD_mt}).

\begin{thm}\label{thm: utility-privacy trade-off_JSD}
Define the privacy leakage of any parameter $w$ released by client $k$ as $\epsilon_{p,w,k} = \log\left(\frac{f_{D_k|W_k}(d|w)}{f_{D_k}(d)}\right)$. Let \pref{assump: defi_of_Delta} hold, and $\epsilon_{p,k}$ and $\epsilon_{u}$ be defined in \pref{defi: average_privacy_dtv_app} and \pref{defi: utility_loss}. Let $P^{o}_k$ and $P^d_k$ represent the distribution of the parameter of client $k$ before and after being protected. Let $F^{\calB}_k$ and $F^{\calA}_k$ represent the belief of client $k$ about $D$ before and after observing the released parameter. Then we have
\begin{align}
 \frac{1}{K}\sum_{k=1}^K JS_{\alpha}(F^{\star}_k || F^{\calO}_k)^{1/e}\le\frac{1}{K}\sum_{k=1}^K \epsilon_{p,k} + 2\gamma\alpha(1-\alpha)(e^{2\delta}-1)\epsilon_u/\Delta,
\end{align}
where $\gamma = \frac{\frac{1}{K}\sum_{k=1}^K ||P^{o}_k - P^d_k||^{1/e}_{\text{TV}}}{||P^d - P^{o}||_{\text{TV}}}$.
\end{thm}

\begin{proof}
Let $T$ represent the round when the algorithm starts to converge. From \pref{lem: total_variation-privacy trade-off_GJSD}, we have

\begin{align}\label{eq: converge_eq_1_modify_1}
   \frac{1}{K}\sum_{k=1}^K JS_{\alpha}(F^{\star}_k || F^{\calO}_k)^{1/e}\le\frac{1}{K}\sum_{k=1}^K \epsilon_{p,k} + \frac{1}{K}\sum_{k=1}^K \left[2\alpha(1-\alpha)(e^{2\delta}-1)||P^{o}_k - P^{d}_k||_{\text{TV}}\right]^{1/e}.
\end{align}

From \pref{lem: total_variation-utility trade-off}, we have
\begin{align}\label{eq: converge_eq_2}
    \epsilon_{u} \ge \frac{1}{2K}\sum_{k=1}^K \Delta_k||P^d_a - P^{o}_a||_{\text{TV}}.
\end{align}

Combining \pref{eq: converge_eq_1_modify_1} and \pref{eq: converge_eq_2}, we have that

\begin{align*}
 \frac{1}{K}\sum_{k=1}^K JS_{\alpha}(F^{\star}_k || F^{\calO}_k)^{1/e}\le\overline\epsilon_{p} + 2\gamma[2\alpha(1-\alpha)(e^{2\delta}-1)]^{1/e}\epsilon_u/\overline\Delta,
\end{align*}
where $\gamma = \frac{\frac{1}{K}\sum_{k=1}^K ||P^{o}_k - P^d_k||^{1/e}_{\text{TV}}}{||P^d - P^{o}||_{\text{TV}}}$, $\overline\Delta = \sum_{k=1}^K \Delta_k$, and $\overline\epsilon_{p} = \frac{1}{K}\sum_{k=1}^K \epsilon_{p,k}$.

This equation could be simplified as
\begin{align*}
    C_1 &\le\frac{1}{K}\sum_{k=1}^K \epsilon_{p,k} + C_2\epsilon_u.
\end{align*}

\end{proof}

\section{Measuring privacy leakage using total variation distance}\label{sec:tvd PL}

The main result generalizes to the case where distance are measured using the total variation distance. 
\subsection{The privacy leakage}
The average privacy leakage is calculated by taking average over all possible assignments of $d$ and $w$.
First we give a definition of average privacy leakage:
\begin{defi}\label{defi: average_privacy_dtv_app}
The average privacy leakage is defined as
\begin{align}
    \epsilon_{p,k} = \int_{d\in\mathcal D_k}p_\calD(d)|f^{\calA}_{D_k}(d) - f_{D_k}(d)|d\mu(d),
\end{align}
where $f^{\calA}_{D_k}(d) = \int_{\mathcal{W}_k} f_{D_k|W_k}(d|w)dP^d_k(w)$.
\end{defi}

\subsection{The Theoretical Analysis}\label{sec: sketch_of_the_analysis}

\subsubsection{The quantitative relationship between $||P^d_k - P^{o}_k||_{\text{TV}}$ and $\delta$}

\begin{lem}\label{lem: dtv_bound}

Let $f^{\calA}_D(d) = \int_{\mathcal{W}_k} f_{D_k|W_k}(d|w)dP^d_k(w)$, and $F^{\star}_k(d) = \int_{\mathcal{W}_k} f_{D_k|W_k}(d|w)dP^{o}_k(w)$, where $D$ represents the private information, and $W$ represents the parameter. For any $\delta\ge 0$, let $f_{W|D}(\cdot|\cdot)$ be a privacy preserving mapping that guarantees $\delta$-maximum Bayesian privacy. That is,
\begin{align*}
    |f_{D_k|W_k}(d|w) - f_{D_k}(d)|\le \delta,
\end{align*}
for any $w\in\mathcal{W}_k$ and any $d\in\mathcal \calD_k$. 
Then, we have
\begin{align*}
    ||F^{\calA}_k - F^{\star}_k||_{\text{TV}} \le 2\delta||P^{o}_k - P^d_k||_{\text{TV}}.
\end{align*}
\end{lem}

\begin{proof}
Let $\mathcal U_k = \{w\in\mathcal{W}_k: dP^{d}_k(w) - dP^{o}_k(w)>0\}$, and $\mathcal V_k = \{w\in\mathcal{W}_k: dP^{d}_k(w) - dP^{o}_k(w)\le 0\}$.
\begin{align}\label{eq:initial_step}
    \left|f^{\calA}_{D_k}(d) - f^{\star}_{D_k}(d)\right| &= \left|\int_{\mathcal{W}_k} f_{D_k|W_k}(d|w)[d P^{d}_k(w) - d P^{o}_k(w)]\right|\nonumber\\
    &= \left|\int_{\mathcal{U}_k} f_{D_k|W_k}(d|w)[d P^{d}_k(w) - d P^{o}_k(w)] + \int_{\mathcal{V}_k} f_{D_k|W_k}(d|w)[d P^{d}_k(w) - d P^{o}_k(w)]\right|\nonumber\\
    &\le\sup_{w\in\mathcal U_k} f_{D_k|W_k}(d|w)\int_{\U} [d P^{d}_k(w) - d P^{o}_k(w)] + \inf_{w\in\mathcal V_k} f_{D_k|W_k}(d|w)\int_\V[d P^{d}_k(w) - d P^{o}_k(w)]\nonumber\\
    &\le\left(\sup_{w\in\mathcal{W}_k} f_{D_k|W_k}(d|w) - \inf_{w\in\mathcal{W}_k} f_{D_k|W_k}(d|w)\right)\int_\U [d P^{d}_k(w) - d P^{o}_k(w)]
\end{align}
where the fourth equality is due to $\mathcal V\subset W^\star$, and the last equality is due to the definition of $w^{\star}$.
First, we provide an upper bound for $(\sup_{w} f_{D_k|W_k}(d|w) - \inf_{w} f_{D_k|W_k}(d|w))$.

From the definition of maximum privacy leakage, we know that for any $w\in\mathcal{W}_k$,
\begin{align*}
    f_{D_k|W_k}(d|w) - f_D(d)\le \delta,
\end{align*}
and
\begin{align*}
    f_{D_k|W_k}(d|w) - f_D(d)\ge -\delta.
\end{align*}
\begin{align}\label{eq: bound_1_term_1_new}
    \sup_{w\in\mathcal{W}_k} f_{D_k|W_k}(d|w) - \inf_{w\in\mathcal{W}_k}f_{D_k|W_k}(d|w) &= \sup_{w\in\mathcal{W}_k} f_{D_k|W_k}(d|w) - f_D(d) - (\inf_{w\in\mathcal{W}_k}f_{D_k|W_k}(d|w) - f_D(d))\nonumber\\
    &\le 2\delta.
\end{align}

From the definition of total variation distance, we have
\begin{align}\label{eq: bound_1_term_2_new}
    \int_\U [d P^{d}_k(w) - d P^{o}_k(w)] = ||P^{o}_k - P^d_k||_{\text{TV}}.  
\end{align}

Combining \pref{eq: bound_1_term_1_new} and \pref{eq: bound_1_term_2_new}, we have
\begin{align}\label{eq: bound_for_the_gap_new}
        |f^{\calA}_{D_k}(d) - f^{\star}_{D_k}(d)| &\le\left(\sup_{w\in\mathcal{W}_k} f_{D_k|W_k}(d|w) - \inf_{w\in\mathcal{W}_k} f_{D_k|W_k}(d|w)\right)\int_\U [d P^{d}_k(w) - d P^{o}_k(w)]\nonumber\\
        &\le 2\delta\cdot ||P^{O}_k - P^d_k||_{\text{TV}}.
\end{align}

Therefore, we have
\begin{align*}
    ||F^{\calA}_k - F^{\star}_k||_{\text{TV}} & = \int_{\mathcal{D}_k} \left|(f^{\calA}_{D_k}(d) - f^{\star}_{D_k}(d))\right|d\mu(d)\\
    &\le 2\delta||P^{o}_k - P^d_k||_{\text{TV}}.
\end{align*}

\end{proof}

\begin{lem}\label{lem: total_variation-privacy trade-off_dtv}
Let $f^{\calA}_{D_k}(d) = \int_{\mathcal{W}_k} f_{D_k|W_k}(d|w)dP^d_k(w)$, and $f^{\star}_{D_k}(d) = \int_{\mathcal{W}_k} f_{D_k|W_k}(d|w)dP^{o}_k(w)$, where $D$ represents the privation information, and $W$ represents the parameter. For any $\delta\ge 0$, let $f_{W|D}(\cdot|\cdot)$ be a privacy preserving mapping that guarantees $\delta$-maximum Bayesian privacy. That is,
\begin{align*}
    |f_{D_k|W_k}(d|w) - f_{D_k}(d)|\le \delta,
\end{align*}
for any $w\in\mathcal{W}_k$ and any $d\in\mathcal \calD_k$.
Then, we have
\begin{align}\label{eq: total_variation-privacy trade-off_dtv}
    ||F^{\star}_k - F^{\calO}_k||_{\text{TV}}\le \epsilon_{p,k} +  2\delta||P^{o}_k - P^d_k||_{\text{TV}}.
\end{align}
\end{lem}

\begin{proof}
From \pref{lem: dtv_bound}, we know that

\begin{align*}
    ||F^{\calA}_k - F^{\star}_k||_{\text{TV}} \le2\delta||P^{o}_k - P^d_k||_{\text{TV}}.
\end{align*}

Notice that the total variation distance satisfies the triangle inequality. From the triangle inequality, we have that
\begin{align*}
    ||F^{\star}_k - F^{\calO}_k||_{\text{TV}} - ||F^{\calA}_k - F^{\calO}_k||_{\text{TV}}\le ||F^{\calA}_k - F^{\star}_k||_{\text{TV}}\le 2\delta||P^{o}_k - P^d_k||_{\text{TV}}.
\end{align*}

Therefore, we have that

\begin{align*}
    ||F^{\star}_k - F^{\calO}_k||_{\text{TV}} &\le ||F^{\calA}_k - F^{\calO}_k||_{\text{TV}} + ||F^{\calA}_k - F^{\star}_k||_{\text{TV}}\\
    &\le \epsilon_{p,k} + 2\delta||P^{o}_k - P^d_k||_{\text{TV}}.
\end{align*}
\end{proof}

\subsection{Analysis of \pref{thm: utility-privacy trade-off_TVD_mt}}

With \pref{lem: total_variation-privacy trade-off_dtv} and \pref{lem: total_variation-utility trade-off}, it is now natural to provide a quantitative relationship between the utility loss and the privacy leakage (\pref{thm: utility-privacy trade-off_TVD_mt}).

\begin{thm}\label{thm: utility-privacy trade-off_dtv}
Let \pref{assump: defi_of_Delta} hold, and $\epsilon_{p,k}$, $\epsilon_{u}$ be defined in \pref{defi: average_privacy_dtv_app} and \pref{defi: utility_loss}. Then we have
\begin{align*}
    \frac{1}{K}\sum_{k=1}^K ||F^{\star}_k - F^{\calO}_k||_{\text{TV}} \le\frac{1}{K}\sum_{k=1}^K \epsilon_{p,k} + \frac{\delta}{\Delta}(e^{2\delta}-1)\epsilon_u.
\end{align*}
\end{thm}

\begin{proof}
Let $T$ represent the round when the algorithm starts to converge. From \pref{lem: total_variation-privacy trade-off_dtv}, we have

\begin{align}\label{eq: converge_eq_1_modify_1_dtv}
    \frac{1}{K}\sum_{k=1}^K ||F^{\star}_k - F^{\calO}_k||_{\text{TV}}\le \frac{1}{K}\sum_{k=1}^K \epsilon_{p,k} +  \frac{1}{K}\sum_{k=1}^K 2\delta||P^{o}_k - P^d_k||_{\text{TV}}.
\end{align}

From \pref{lem: total_variation-utility trade-off}, we have
\begin{align}\label{eq: converge_eq_2_dtv}
    \epsilon_{u} \ge \frac{1}{2K}\sum_{k=1}^K \Delta_k||P^d_a - P^{o}_a||_{\text{TV}}.
\end{align}

Combining \pref{eq: converge_eq_1_modify_1_dtv} and \pref{eq: converge_eq_2_dtv}, we have that 

\begin{align*}
    \frac{1}{K}\sum_{k=1}^K ||F^{\star}_k - F^{\calO}_k||_{\text{TV}} \le\overline\epsilon_{p} + \frac{4\gamma\delta}{\overline\Delta}\epsilon_u,
\end{align*}
where $\gamma = \frac{\frac{1}{K}\sum_{k=1}^K ||P^{o}_k - P^d_k||_{\text{TV}}}{||P^d_a - P^{o}_a||_{\text{TV}}}$, $\overline\Delta = \sum_{k=1}^K \Delta_k$, and $\overline\epsilon_{p} = \frac{1}{K}\sum_{k=1}^K \epsilon_{p,k}$.

\end{proof}

\end{appendix}

\end{document}